\documentclass[runningheads]{llncs}
\usepackage[T1]{fontenc}
\usepackage{graphicx}
\usepackage{booktabs}
\usepackage{multirow}
\usepackage{wrapfig}
\usepackage{amsmath}
\usepackage{amssymb}
\usepackage{amsfonts}
\usepackage{mathtools}
\usepackage{subfig}
\usepackage[section]{placeins}
\usepackage{enumitem}
\usepackage{xspace}

\begin{document}
\title{Data-Efficient Agentic Graph Domain Adaptation via Reliability-Aware Prototype Learning}

\newcommand{\method}{\textsc{DEAG}}

\titlerunning{Data-Efficient Agentic Graph Domain Adaptation}

\author{
Yingxu Wang\inst{1}\textsuperscript{\textdagger}
\and
Kunyu Zhang\inst{2}\textsuperscript{\textdagger}
\and
Siyang Gao\inst{3}\thanks{Corresponding author.}
}

\authorrunning{Y. Wang et al.}

\institute{
Mohamed bin Zayed University of Artificial Intelligence, Abu Dhabi, UAE\\
\email{yingxv.wang@gmail.com}
\and
Zhengzhou University, Zhengzhou, China\\
\email{kunyu.zky@gmail.com}
\and
City University of Hong Kong, Hong Kong SAR, China\\
\email{siyangao@cityu.edu.hk}
}

\maketitle

\begingroup
\renewcommand{\thefootnote}{\textdagger}
\footnotetext{Yingxu Wang and Kunyu Zhang contributed equally to this work.}
\endgroup
\begin{abstract}
Agentic learning systems are often required to adapt after deployment by observing new data and reusing prior knowledge under limited supervision or feedback. For graph-structured prediction, Graph Domain Adaptation (GDA) naturally instantiates this setting by transferring knowledge from labeled source graphs to unlabeled target graphs under distribution shifts. However, most GDA methods assume sufficient labeled source graphs, which becomes restrictive in data-efficient agentic settings where only limited source evidence can be retained. Under such constraints, source semantics become unreliable, leading to unstable source anchoring, uncertain target association, and fragile target-marginal calibration. To address these challenges, we propose \method{}, a reliability-aware prototype learning framework for data-efficient agentic GDA. \method{} estimates class reliability from retained source support and embedding compactness, and constructs stable reusable source anchors by blending empirical prototypes with classifier directions. Guided by these anchors, \method{} performs prototype-aware soft target association and aligns confidence-weighted target centers with source semantics. A source-prior regularizer further sharpens target predictions while keeping the target marginal consistent with retained source evidence. Experiments on graph benchmarks with diverse domain shifts show that \method{} improves average adaptation performance over competitive GDA baselines under the same source-data budget.\\
\end{abstract}

\section{Introduction}

Agentic learning systems are increasingly expected to adapt after deployment by observing new data and reusing prior knowledge under limited supervision or feedback~\cite{hu2025automated,collaco2026role,wang2026sgac}. For graph-structured prediction, Graph Domain Adaptation (GDA) provides a natural formulation of this setting: it transfers predictive knowledge from labeled source graphs to unlabeled target graphs under distribution shifts~\cite{yin2022deal,yin2023coco,hadipour2025graphban}, which enables graph neural networks (GNNs) to generalize beyond the environments where annotations are originally collected~\cite{kipf2016semi,xu2018powerful,lei2025gradual,yang2025graphlora}. Such transferability is valuable in molecular property prediction across chemical libraries and social network mining across platforms~\cite{xu2018powerful,wang2026nested,fang2025homophily,wang2025protomol}.

Existing GDA methods mainly improve transferability by reducing source-target discrepancy in graph representation space~\cite{ganin2016domain,long2015learning,zhang2026brainriem}. Representative approaches adopt adversarial or discrepancy-based alignment, contrastive cross-domain learning, and graph-specific regularization~\cite{ngo2025higda,zeng2025pave,yin2025dream}. Other studies further exploit graph-structured inductive biases through topology-aware constraints, graph reweighting, propagation calibration, or target self-training with structural consistency~\cite{liu2023structural,liu2024rethinking,qiao2023semi}. Despite these advances, most methods are developed in a source-sufficient regime, where the labeled source collection is assumed to be fully accessible during training~\cite{wang2025correction,chen2025smoothness,wang2026disrfm}. This assumption is restrictive for data-efficient agentic adaptation, where a learner may need to update its predictor after deployment using only limited retained source evidence and unlabeled target observations. In practical graph learning pipelines, source graphs may also be constrained by storage budgets, preprocessing overhead, sampling restrictions, or repeated deployment costs~\cite{liang2020we,wang2025dusego,wang2026cross}. This raises an important question: \textit{how can an adaptive graph learner reliably reuse limited source evidence to adapt to unlabeled target graphs under distribution shift?}

Reducing the source-data budget improves practical efficiency, but it also weakens the transfer signals on which GDA relies. In the limited-source regime, directly applying existing GDA objectives leaves three challenges unresolved. \textbf{(1) Unreliable source anchoring.} Alignment-based GDA requires source supervision to provide stable class semantics~\cite{yin2022deal,yin2023coco,yin2025dream}. With limited source graphs, empirical class prototypes can be sensitive to sampling noise, while classifier directions learned from the same reduced data may also be biased~\cite{liang2020we,qiao2023semi}. An adaptive learner should therefore assess the reliability of retained source evidence and combine instance-level prototype information with classifier-level decision information accordingly. \textbf{(2) Uncertain target association.} Since target labels are unavailable, target graphs must be associated with source semantics through model predictions and cross-domain similarity~\cite{yin2022deal,qiao2023semi,liu2024rethinking}. Under weakened source supervision, early predictions are unstable, and hard pseudo-label training can amplify incorrect assignments~\cite{qiao2023semi,liang2020we,wang2024degree}. Target association should therefore remain soft, incorporate source-anchor compatibility, and downweight ambiguous target graphs, so that target observations can be cautiously grounded in reusable source semantics. \textbf{(3) Fragile target-marginal calibration.} Confidence regularization sharpens target predictions, but a weak source model may over-concentrate target samples into a few initially favored classes~\cite{liang2020we,garg2023rlsbench}. Conversely, enforcing a uniform target marginal can conflict with the class evidence retained from limited source data~\cite{garg2023rlsbench,liu2024pairwise}. A robust regularizer should calibrate the target marginal using the available source prior while preserving confident sample-level predictions, thereby reducing biased prediction updates before reliable target feedback is available.

In this paper, we propose \method{}, a reliability-aware prototype learning framework for data-efficient agentic GDA. Rather than directly adapting with weakened source supervision, \method{} stabilizes reusable source-side guidance and supports self-guided adaptation under limited source evidence. It estimates class reliability from retained source support and representation compactness, then constructs reliability-aware anchors by blending empirical source prototypes with classifier directions. Guided by these anchors, \method{} performs prototype-aware soft target association by combining model predictions and target-to-anchor similarities, allowing unlabeled target observations to be cautiously associated with reusable source semantics. The resulting confidence-weighted target centers are aligned with the corresponding source anchors, avoiding hard pseudo-label cross-entropy. Finally, \method{} introduces source-prior regularization to sharpen target predictions while preventing the target marginal from drifting away from the class distribution observed in the retained source data. Experiments on graph benchmarks with diverse domain shifts show that \method{} improves average adaptation performance over competitive GDA baselines under the same source-data budget.

Our contributions can be summarized as follows: (1) We investigate data-efficient agentic GDA under limited source evidence and identify three key challenges: unreliable source anchoring, uncertain target association, and fragile target-marginal calibration. (2) We propose \method{}, a compact reliability-aware prototype learning framework that stabilizes reusable source guidance, performs soft target alignment, and calibrates target predictions with the retained source prior. (3) Experiments on multiple graph benchmarks under diverse distribution shifts demonstrate that \method{} improves average adaptation performance over competitive baselines under the same source-data budget.
\vspace{-0.3cm}
\section{Related Work}

\noindent\textbf{{Data-Efficient Agentic Graph Learning.}}
Data-efficient agentic graph learning aims to train or adapt effective graph learners with limited training data, supervision, feedback, or accessible graph evidence~\cite{ju2024survey}. It is closely related to data-efficient graph learning, where existing studies broadly follow two directions~\cite{ju2024survey,hashemi2024comprehensive,gao2025graph}. The first direction improves label efficiency by leveraging unlabeled graphs, semi-supervised regularization, self-supervised objectives, or selective label acquisition, thereby reducing the dependence on dense annotations~\cite{ju2024survey,tian2025knowledge,hu2025large}. The second direction improves data efficiency by compressing the effective training set through coreset selection, dataset distillation, or graph condensation, where a compact set of real or synthetic graphs is used to approximate the learning effect of the original dataset~\cite{hashemi2024comprehensive,zhang2024navigating,wang2026usbd}. Despite their progress, these methods are mostly developed for single-domain graph learning and do not explicitly address agentic adaptation under source-target distribution shifts~\cite{yin2023coco,yu2025samgpt,ngo2025higda}. This distinction is important for GDA: when source evidence is limited, the retained source graphs not only provide supervised training signals but also define the reusable class semantics used to guide target-domain adaptation. A reduced source set may therefore yield noisy source prototypes, unstable target association, and miscalibrated target predictions~\cite{liang2020we,garg2023rlsbench,yin2025dream}. In contrast, \method{} studies data-efficient agentic graph domain adaptation under limited retained source evidence. It supports self-guided adaptation by estimating source prototype reliability, associating unlabeled target graphs through prototype-aware soft alignment, and calibrating target predictions using the class prior observed from the retained source data.

\begin{figure}[t]
    \centering
    \includegraphics[width=1.0\linewidth]{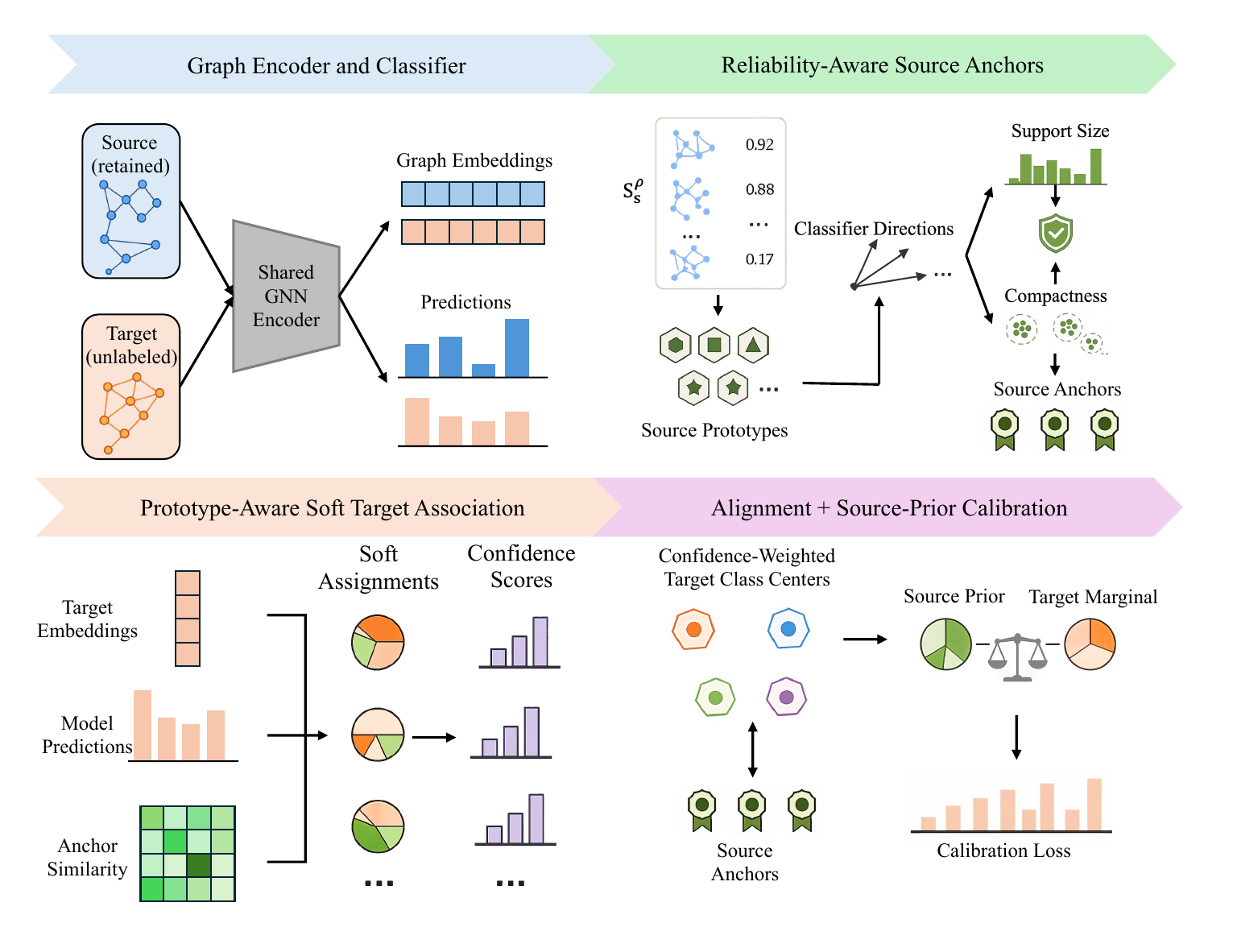}
    \vspace{-0.7cm}
    \caption{Overview of the proposed \method{}. \method{} estimates source-class reliability using support size and embedding compactness, constructs reusable source anchors by blending source prototypes with classifier directions, softly associates target graphs with these anchors, and jointly optimizes prototype alignment and source-prior target-marginal calibration.}
    \label{fig:overview}
    \vspace{-0.6cm}
\end{figure}

\noindent\textbf{{Graph Domain Adaptation (GDA).}}
GDA transfers predictive knowledge from labeled source graphs to unlabeled target graphs under distribution shifts~\cite{you2023graph,cai2024graph,dai2022graph}. Existing GDA methods mainly improve transferability by reducing source-target mismatch in the learned representation space~\cite{wang2025bridging,long2015learning,zhang2024anchor}. Representative approaches employ adversarial objectives, distribution alignment, contrastive learning, or variational modeling to learn domain-adaptive graph representations~\cite{zhang2026survey,li2023sigma}. Since graph data contain both node attributes and relational structures, later studies further incorporate graph-specific inductive biases, including topology-aware regularization, structural reweighting, propagation calibration, and graph-constrained target self-training, to better handle attribute and structural shifts during adaptation~\cite{liu2023structural,liu2024rethinking,chen2025smoothness}. Despite their effectiveness, most existing GDA methods assume that sufficient labeled source graphs are available during training, so that source classifiers, class prototypes, and alignment signals can be estimated reliably~\cite{yin2022deal,yin2025dream}. This source-sufficient assumption becomes fragile when the source-data budget is constrained, especially when an adaptive learner must rely on limited retained source evidence for post-deployment adaptation~\cite{liang2020we}. With fewer source graphs, empirical class prototypes may be biased by sampling noise, target predictions become less reliable, and hard pseudo-labeling or unconstrained alignment can amplify early errors~\cite{qiao2023semi,liang2020we,garg2023rlsbench}. To address this limitation, \method{} learns reliability-aware source anchors by fusing source prototypes with classifier directions, and uses them for soft target association and source-prior target-marginal calibration, enabling self-guided adaptation under limited source evidence.
\section{Methodology}

\subsection{Problem Formulation}

We study data-efficient agentic Graph Domain Adaptation (GDA) under limited retained source evidence. Let $\mathcal{D}_S=\{(\mathcal{G}_i^s,y_i^s)\}_{i=1}^{N_S}$ be the labeled source domain and $\mathcal{D}_T=\{\mathcal{G}_j^t\}_{j=1}^{N_T}$ be the unlabeled target domain, where each graph is $\mathcal{G}=(\mathcal{V},\mathcal{E},\mathbf{X})$ and $y_i^s\in\mathcal{Y}=\{1,\ldots,C\}$. Given a source-data ratio $\rho\in(0,1]$, only a class-stratified retained subset $\mathcal{S}_S^\rho=\bigcup_{c=1}^{C}\mathcal{S}_{S,c}^\rho$ is accessible for training, where $\mathcal{S}_{S,c}^\rho\subseteq\mathcal{D}_{S,c}$, $\mathcal{D}_{S,c}=\{(\mathcal{G}_i^s,y_i^s)\in\mathcal{D}_S:y_i^s=c\}$, and $|\mathcal{S}_{S,c}^\rho|=K_c$. We denote the retained source budget by $K=\sum_{c=1}^{C}K_c\ll N_S$.

\subsection{Overview of \method{}}

Figure~\ref{fig:overview} illustrates the proposed framework of \method{}. It first extracts graph embeddings with a shared GNN encoder and estimates the reliability of each source class from the retained source subset. To obtain reusable source-side guidance, \method{} blends empirical prototypes with classifier directions according to this reliability, producing reliability-aware source anchors. These anchors then guide the association of unlabeled target observations by combining classifier predictions with target-to-anchor similarities. The resulting soft assignments are confidence-weighted to form target class centers, which are aligned with source anchors without hard pseudo-label training. Finally, a source-prior regularizer calibrates the target marginal using the class prior observed from the retained source evidence.

\subsection{Reliability-Aware Source Anchors}

\noindent\textbf{Encoder and classifier.}
We use a standard GNN encoder followed by a linear classifier. For a graph $\mathcal{G}$, node representations are updated by
\begin{equation}
    \mathbf{H}_{\mathcal{G}}^{(\ell)}
    =
    \mathrm{GNN}_{\theta}^{(\ell)}
    (\mathcal{G},\mathbf{H}_{\mathcal{G}}^{(\ell-1)}),
    \qquad
    \mathbf{H}_{\mathcal{G}}^{(0)}=\mathbf{X}.
    \label{eq:gnn_update}
\end{equation}
The graph representation is obtained as
\begin{equation}
    \mathbf{z}_{\mathcal{G}}
    =
    \mathrm{READOUT}\left(\mathbf{H}_{\mathcal{G}}^{(L)}\right)
    \in\mathbb{R}^{d_z},
    \label{eq:graph_embedding}
\end{equation}
and its normalized version is $\widetilde{\mathbf{z}}_{\mathcal{G}}=\mathbf{z}_{\mathcal{G}}/(\|\mathbf{z}_{\mathcal{G}}\|_2+\varepsilon)$, where $\varepsilon>0$ is a numerical constant. The classifier predicts
\begin{equation}
    p_{\theta}(y|\mathcal{G})
    =
    \mathrm{softmax}(\mathbf{W}\mathbf{z}_{\mathcal{G}}+\mathbf{b}).
    \label{eq:classifier}
\end{equation}
Let $\mathbf{w}_c$ be the classifier weight vector for class $c$, and define its normalized direction as $\widetilde{\mathbf{w}}_c=\mathbf{w}_c/(\|\mathbf{w}_c\|_2+\varepsilon)$.

\noindent\textbf{Source prototype reliability.}
Under limited retained source evidence, empirical class prototypes may be noisy and may not provide stable reusable source semantics. For each class $c$, we compute the retained-source prototype as
\begin{equation}
    \bar{\boldsymbol{\mu}}_{S,c}
    =
    \frac{1}{K_c}
    \sum_{(\mathcal{G}_i^s,y_i^s)\in\mathcal{S}_{S,c}^{\rho}}
    \widetilde{\mathbf{z}}_{\mathcal{G}_i^s},
    \qquad
    \boldsymbol{\mu}_{S,c}
    =
    \frac{\bar{\boldsymbol{\mu}}_{S,c}}{\|\bar{\boldsymbol{\mu}}_{S,c}\|_2+\varepsilon}.
    \label{eq:source_proto}
\end{equation}
We measure prototype compactness by $r_c=\|\bar{\boldsymbol{\mu}}_{S,c}\|_2$. Since each normalized embedding has unit norm, $0\le r_c\le1$; larger $r_c$ indicates more concentrated retained source embeddings. We combine compactness with source support size to define the reliability coefficient
\begin{equation}
    \alpha_c
    =
    \mathrm{sg}\left[
    \frac{K_c-1}{K_c}r_c
    \right],
    \label{eq:reliability}
\end{equation}
where $\mathrm{sg}[\cdot]$ denotes stop-gradient. Thus, a class with a single retained source graph obtains $\alpha_c=0$, while a class with sufficient and compact samples obtains a larger reliability score. This coefficient allows the learner to assess how trustworthy the retained source evidence is for each class.

\noindent\textbf{Anchor construction.}
To balance instance-level source semantics and classifier-level decision geometry, we fuse the empirical prototype with the classifier direction:
\begin{equation}
    \bar{\mathbf{a}}_c
    =
    \alpha_c\boldsymbol{\mu}_{S,c}
    +
    (1-\alpha_c)\widetilde{\mathbf{w}}_c,
    \qquad
    \mathbf{a}_c
    =
    \mathrm{sg}\left[
    \frac{\bar{\mathbf{a}}_c}{\|\bar{\mathbf{a}}_c\|_2+\varepsilon}
    \right].
    \label{eq:source_anchor}
\end{equation}
The anchor set is $\mathcal{A}_S=\{\mathbf{a}_c\}_{c=1}^{C}$. When the prototype is reliable, $\mathbf{a}_c$ is dominated by $\boldsymbol{\mu}_{S,c}$; otherwise, it falls back to the classifier direction. This prevents noisy prototypes from becoming misleading alignment targets and provides stable reusable source guidance for target adaptation.

\subsection{Prototype-Aware Target Association}

Since target labels are unavailable, target graphs must be associated with source semantics through predictions and representation similarity. Hard pseudo-labeling is fragile under limited source evidence, so \method{} constructs soft and confidence-weighted target class centers. For target graph $\mathcal{G}_j^t$, let $\widetilde{\mathbf{z}}_j^t=\widetilde{\mathbf{z}}_{\mathcal{G}_j^t}$ and $p_{j,c}^t=p_{\theta}(c|\mathcal{G}_j^t)$. We define the prototype-aware assignment logit as
\begin{equation}
    \ell_{j,c}
    =
    \log(p_{j,c}^{t}+\varepsilon)
    +
    (\widetilde{\mathbf{z}}_j^t)^{\top}\mathbf{a}_c.
    \label{eq:target_logit}
\end{equation}
The first term captures classifier confidence, and the second term measures target-to-anchor compatibility. The detached soft assignment is
\begin{equation}
    q_{j,c}
    =
    \mathrm{sg}\left[
    \frac{\exp(\ell_{j,c})}{\sum_{k=1}^{C}\exp(\ell_{j,k})}
    \right].
    \label{eq:soft_assignment}
\end{equation}
Here, $q_{j,c}$ is not used as a target pseudo-label for cross-entropy training; it only provides weights for target center estimation. This design allows unlabeled target observations to be cautiously associated with reusable source anchors without committing to hard pseudo-labels.

To suppress ambiguous target graphs, we compute the assignment entropy $H(q_j)=-\sum_{c=1}^{C}q_{j,c}\log(q_{j,c}+\varepsilon)$ and define
\begin{equation}
    \omega_j
    =
    \mathrm{sg}\left[
    1-\frac{H(q_j)}{\log C}
    \right].
    \label{eq:confidence}
\end{equation}
For class $c$, the weighted target mass and detached class weight are
\begin{equation}
    M_c=\sum_{j=1}^{N_T}\omega_jq_{j,c},
    \qquad
    \beta_c =
    \mathrm{sg}\left[
    \frac{M_c+\varepsilon}{\sum_{k=1}^{C}M_k+C\varepsilon}
    \right].
    \label{eq:target_mass}
\end{equation}
The normalized target center is
\begin{equation}
    \bar{\boldsymbol{\mu}}_{T,c}
    =
    \frac{\sum_{j=1}^{N_T}\omega_jq_{j,c}\widetilde{\mathbf{z}}_j^t}{M_c+\varepsilon},
    \qquad
    \boldsymbol{\mu}_{T,c}
    =
    \frac{\bar{\boldsymbol{\mu}}_{T,c}}{\|\bar{\boldsymbol{\mu}}_{T,c}\|_2+\varepsilon}.
    \label{eq:target_center}
\end{equation}
We align target centers to source anchors by
\begin{equation}
    \mathcal{L}_{pa}
    =
    \sum_{c=1}^{C}
    \beta_c
    \left\|
    \boldsymbol{\mu}_{T,c}-\mathbf{a}_c
    \right\|_2^2.
    \label{eq:prototype_alignment}
\end{equation}
Gradients are propagated through $\boldsymbol{\mu}_{T,c}$ to update target representations, while $\mathbf{a}_c$, $q_{j,c}$, $\omega_j$, and $\beta_c$ are detached. This avoids degenerate minimization by changing anchors or assignments instead of learning target-aligned representations.

\subsection{Source-Prior Target-Marginal Calibration}

Prototype alignment transfers reusable class-level source semantics to the target domain, but it may over-concentrate target graphs around a few anchors. To calibrate target predictions before reliable target feedback is available, \method{} introduces a source-prior regularizer. The retained source prior is $\pi_{S,c}=K_c/K$, and the target marginal prediction is $\bar{p}_{T,c}=N_T^{-1}\sum_{j=1}^{N_T}p_{\theta}(c|\mathcal{G}_j^t)$. The target conditional entropy is
\begin{equation}
    \mathcal{H}_{cond}
    =
    -
    \frac{1}{N_T}
    \sum_{j=1}^{N_T}
    \sum_{c=1}^{C}
    p_{\theta}(c|\mathcal{G}_j^t)
    \log\left(p_{\theta}(c|\mathcal{G}_j^t)+\varepsilon\right),
    \label{eq:target_entropy}
\end{equation}
and the source-prior matching term is
\begin{equation}
    \mathcal{D}_{prior}
    =
    \mathrm{KL}(\pi_S\Vert\bar{p}_T)
    =
    \sum_{c=1}^{C}
    \pi_{S,c}
    \log\left(
    \frac{\pi_{S,c}+\varepsilon}{\bar{p}_{T,c}+\varepsilon}
    \right).
    \label{eq:prior_matching}
\end{equation}
The source-prior target regularization loss is $\mathcal{L}_{ptr}=\mathcal{H}_{cond}+\mathcal{D}_{prior}$. The entropy term sharpens target predictions, while the prior matching term prevents the target marginal from drifting away from retained source evidence. Together, they reduce biased target-marginal updates under limited source evidence and unavailable target labels.

\subsection{Overall Learning Objective}
\label{sec:overall_objective}

The supervised loss is computed only on $\mathcal{S}_S^\rho$. To reduce class imbalance under limited retained source evidence, we use class-balanced cross-entropy:
\begin{equation}
    \mathcal{L}_{sup}
    =
    \frac{1}{C}
    \sum_{c=1}^{C}
    \frac{1}{K_c}
    \sum_{(\mathcal{G}_i^s,y_i^s)\in\mathcal{S}_{S,c}^{\rho}}
    \mathcal{L}_{CE}
    \left(
    p_{\theta}(y|\mathcal{G}_i^s),c
    \right).
    \label{eq:source_supervised}
\end{equation}
The final objective is
\begin{equation}
    \mathcal{L}
    =
    \mathcal{L}_{sup}
    +
    \lambda_1\mathcal{L}_{pa}
    +
    \lambda_2\mathcal{L}_{ptr},
    \label{eq:overall_objective}
\end{equation}
where $\lambda_1,\lambda_2\ge0$ balance the two adaptation losses. 

\begin{theorem}[Target-Risk Bound]
\label{thm:generalization}
Let $f$ be the graph encoder, $h$ be the classifier, and $g=h\circ f\in\mathcal{H}$ be the learned predictor. Let $\mathcal{R}_{T}(g)$ denote the expected target risk and $\widehat{\mathcal{R}}_{\mathcal{S}_S^{\rho}}(g)$ denote the empirical risk on the retained source subset $\mathcal{S}_S^{\rho}$ with $K$ samples. For each class $c$, let $\mathbf{a}_c$ be the learned reliability-aware source anchor and $\mathbf{a}_c^{\star}$ be the ideal class anchor estimated with sufficient source support. Define the retained-source prior $\pi_{S,c}=K_c/K$, the anchor estimation error $\epsilon_{\mathrm{anc}}=\sum_{c=1}^{C}\pi_{S,c}\|\mathbf{a}_c-\mathbf{a}_c^{\star}\|_2$, and the class-mass factor $\gamma_{\beta}=\sum_{c=1}^{C}\pi_{S,c}^{2}/\beta_c$, where $\beta_c>0$ is the detached target class-mass weight in $\mathcal{L}_{pa}$. Assume the loss $\ell$ is bounded in $[0,1]$ and $K_{\mathrm{Lip}}$-Lipschitz with respect to normalized graph representations. Then, with probability at least $1-\delta$, the target risk is bounded by
\begin{equation}
\label{eq:generalization_bound}
\begin{aligned}
\mathcal{R}_{T}(g)
\leq\;&
\widehat{\mathcal{R}}_{\mathcal{S}_S^{\rho}}(g)
+
K_{\mathrm{Lip}}\sqrt{\gamma_{\beta}\mathcal{L}_{pa}}
+
\sqrt{2\mathcal{D}_{prior}}
+
K_{\mathrm{Lip}}\epsilon_{\mathrm{anc}}
\\
&+
2\mathfrak{R}_{K}(\mathcal{H})
+
3\sqrt{\frac{\log(2/\delta)}{2K}}
+
\lambda^{\star},
\end{aligned}
\end{equation}
where $\mathfrak{R}_{K}(\mathcal{H})$ is the empirical Rademacher complexity of $\mathcal{H}$ under $K$ retained source samples, and $\lambda^{\star}$ denotes the minimal joint risk and residual mismatch between the retained source and target domains.
\end{theorem}

Theorem~\ref{thm:generalization} shows that \method{} controls target risk through three optimizable terms. The prototype alignment loss $\mathcal{L}_{pa}$ reduces class-wise representation discrepancy, the prior matching term $\mathcal{D}_{prior}$ controls target marginal mismatch, and the reliability-aware anchor construction reduces $\epsilon_{\mathrm{anc}}$ by preventing noisy empirical prototypes from dominating target alignment. The proof of Theorem~\ref{thm:generalization} can be found in Appendix~\ref{app:proof_generalization}.

\section{Experiments}

\subsection{Experimental Settings}

\noindent\textbf{Datasets.} We evaluate \method{} on two representative types of graph distribution shifts. (1) \textit{Structure-based domain shifts} are constructed on Mutagenicity~\cite{kazius2005derivation} and NCI1~\cite{wale2008comparison}. Following standard GDA protocols~\cite{yin2022deal,yin2023coco}, we partition graphs by node-density and edge-density statistics to create source and target domains with different structural biases while preserving the original classification tasks. (2) \textit{Feature-based domain shifts} are evaluated on commonly used paired benchmarks, including DD, PROTEINS, BZR, BZR\_MD, COX2, and COX2\_MD, where source and target domains exhibit notable feature distribution discrepancies. The statistics of the above datasets are provided in Appendix~\ref{sec:dataset}.

\noindent\textbf{Baselines.} We compare \method{} with three groups of baselines. The first group contains the graph kernel method PathNN~\cite{michel2023path}. The second group includes four general GNN backbones: GCN~\cite{kipf2016semi}, GIN~\cite{xu2018powerful}, CIN~\cite{bodnar2021weisfeiler}, and GMT~\cite{baek2021accurate}. The third group includes seven representative GDA methods: DEAL~\cite{yin2022deal}, SGDA~\cite{qiao2023semi}, StruRW~\cite{liu2023structural}, A2GNN~\cite{liu2024rethinking}, PA-BOTH~\cite{liu2024pairwise}, GAA~\cite{fang2025benefits}, and TDSS~\cite{chen2025smoothness}.

\noindent\textbf{Implementation Details.} We implement \method{} in PyTorch. Unless otherwise specified, GIN~\cite{xu2018powerful} is used as the backbone encoder, with three GNN layers and a hidden dimension of 128. All neural methods are optimized with Adam using a learning rate of $3\times10^{-4}$ and a weight decay of $1\times10^{-12}$. We use a class-stratified retained source subset with source-data ratio $\rho=0.3$ by default. For \method{}, the loss coefficients are fixed to $\lambda_1=0.7$ and $\lambda_2=0.5$. We report classification accuracy as the evaluation metric, and present the mean and standard deviation over five independent runs.

\begin{table}[t]
\centering
\caption{Graph classification results (in \%) under node and edge density domain shifts on the Mutagenicity dataset, and 
under feature shifts (source$\rightarrow$target). P, D, C, CM, B, and BM denote PROTEINS, DD, COX2, COX2\_MD, BZR, and BZR\_MD, respectively. \textbf{Bold} indicates the best performance.}
\label{tab:main_mutag_updated}
\vspace{0.2cm}
\resizebox{\textwidth}{!}{%
\begin{tabular}{lccc|ccc|cccccc}
\toprule
\multirow{2}{*}{Methods} & \multicolumn{3}{c|}{Node Shift} & \multicolumn{3}{c|}{Edge Shift} & \multicolumn{6}{c}{Feature Shift} \\
\cmidrule(lr){2-4} \cmidrule(lr){5-7} \cmidrule(lr){8-13}
& M0$\rightarrow$M1 & M0$\rightarrow$M2 & M0$\rightarrow$M3 & M0$\rightarrow$M1 & M0$\rightarrow$M2 & M0$\rightarrow$M3 & P$\rightarrow$D & D$\rightarrow$P & C$\rightarrow$CM & CM$\rightarrow$C & B$\rightarrow$BM & BM$\rightarrow$B \\
\midrule
GCN     & 53.2$\pm$4.1 & 49.5$\pm$3.6 & 50.1$\pm$4.5 & 61.5$\pm$5.4 & 53.2$\pm$4.8 & 47.3$\pm$3.2 & 52.3$\pm$4.1 & 48.7$\pm$5.2 & 41.2$\pm$3.8 & 45.6$\pm$2.9 & 44.3$\pm$6.2 & 43.8$\pm$4.7 \\
GIN     & 51.4$\pm$5.2 & 47.8$\pm$4.4 & 49.6$\pm$3.8 & 59.8$\pm$6.1 & 51.6$\pm$5.2 & 46.5$\pm$4.1 & 50.7$\pm$5.3 & 47.2$\pm$4.6 & 40.8$\pm$4.1 & 44.9$\pm$3.2 & 43.6$\pm$5.8 & 42.1$\pm$4.9 \\
GMT     & 55.1$\pm$3.8 & 51.2$\pm$3.2 & 51.5$\pm$2.9 & 62.4$\pm$4.5 & 54.5$\pm$3.9 & 48.2$\pm$2.8 & 54.2$\pm$3.7 & 49.6$\pm$3.4 & 42.4$\pm$3.6 & 46.8$\pm$2.8 & 45.4$\pm$5.1 & 44.7$\pm$4.2 \\
CIN     & 54.3$\pm$4.5 & 50.6$\pm$3.7 & 50.9$\pm$3.4 & 61.9$\pm$5.0 & 53.8$\pm$4.2 & 47.8$\pm$3.1 & 53.6$\pm$4.2 & 48.9$\pm$4.1 & 41.7$\pm$3.9 & 45.2$\pm$3.1 & 44.8$\pm$5.4 & 43.9$\pm$4.5 \\
PathNN  & 52.8$\pm$4.8 & 48.9$\pm$4.1 & 50.4$\pm$3.6 & 60.5$\pm$5.6 & 52.4$\pm$4.9 & 46.9$\pm$3.7 & 51.8$\pm$4.6 & 47.8$\pm$4.3 & 41.1$\pm$4.2 & 44.7$\pm$3.4 & 43.9$\pm$5.6 & 43.2$\pm$4.8 \\
\midrule
DEAL    & 72.1$\pm$2.1 & 64.8$\pm$2.4 & 56.1$\pm$1.2 & 72.8$\pm$0.7 & 62.5$\pm$2.1 & 54.2$\pm$1.4 & 59.8$\pm$0.9 & 50.2$\pm$9.8 & 42.1$\pm$6.1 & 49.3$\pm$3.5 & 48.9$\pm$7.6 & 47.7$\pm$5.0 \\
SGDA    & 54.7$\pm$4.6 & 50.2$\pm$3.7 & 50.8$\pm$4.1 & 63.3$\pm$6.9 & 54.1$\pm$5.7 & 47.9$\pm$4.1 & 48.0$\pm$3.8 & 54.6$\pm$10.8 & 41.3$\pm$0.3 & 49.8$\pm$1.5 & 49.8$\pm$9.9 & 51.7$\pm$8.5 \\
A2GNN   & 42.5$\pm$7.1 & 46.1$\pm$6.5 & 50.6$\pm$3.8 & 60.6$\pm$8.7 & 54.4$\pm$7.2 & 49.1$\pm$3.5 & 61.9$\pm$0.7 & 57.0$\pm$1.0 & 45.4$\pm$9.9 & 45.5$\pm$5.4 & 48.1$\pm$5.3 & 63.0$\pm$4.4 \\
StruRW  & 74.3$\pm$1.7 & 66.8$\pm$1.1 & 53.2$\pm$1.8 & 73.1$\pm$0.9 & 63.7$\pm$1.6 & 50.8$\pm$2.1 & 54.0$\pm$6.9 & 47.6$\pm$11.4 & 40.7$\pm$0.2 & 49.5$\pm$2.8 & 48.8$\pm$8.0 & 52.3$\pm$4.8 \\
PA-BOTH & 72.9$\pm$0.8 & 65.4$\pm$2.3 & 52.1$\pm$0.9 & 72.4$\pm$1.1 & 53.9$\pm$6.2 & 52.5$\pm$1.4 & 54.1$\pm$6.4 & 47.1$\pm$10.0 & 40.7$\pm$0.2 & 49.8$\pm$2.6 & 49.9$\pm$7.0 & 51.4$\pm$4.9 \\
GAA     & 75.8$\pm$1.4 & 68.9$\pm$0.9 & 55.5$\pm$1.1 & 72.9$\pm$1.2 & 65.7$\pm$1.1 & 55.5$\pm$2.3 & 47.0$\pm$3.7 & 58.0$\pm$1.8 & 53.0$\pm$12.8 & 50.1$\pm$3.3 & 52.0$\pm$6.8 & 49.4$\pm$9.7 \\
TDSS    & 43.8$\pm$6.6 & 50.1$\pm$7.1 & 50.9$\pm$3.2 & 62.6$\pm$8.5 & 54.0$\pm$5.9 & 49.4$\pm$4.6 & 62.3$\pm$0.6 & 57.9$\pm$0.5 & 45.4$\pm$9.9 & 46.5$\pm$3.3 & 46.6$\pm$6.9 & 50.5$\pm$24.1 \\
\midrule
\method{}    & \textbf{76.9}$\pm$2.2 & \textbf{69.7}$\pm$1.2 & \textbf{63.4}$\pm$2.3 & \textbf{74.8}$\pm$1.2 & \textbf{67.1}$\pm$1.6 & \textbf{65.6}$\pm$2.3 & \textbf{64.1}$\pm$1.5 & \textbf{62.5}$\pm$3.1 & \textbf{54.2}$\pm$11.8 & \textbf{51.3}$\pm$4.5 & \textbf{53.1}$\pm$9.5 & \textbf{68.3}$\pm$8.8 \\
\bottomrule
\end{tabular}%
}
\vspace{-0.4cm}
\end{table}

\begin{table}[t]
\centering
\caption{Aggregate comparison across all reported tasks. \textbf{Bold} results indicate the best performance.}
\vspace{0.2cm}
\label{tab:aggregate_summary_shift}
\resizebox{0.85\textwidth}{!}{%
\begin{tabular}{lcccccc}
\toprule
Methods 
& Node Avg. 
& Edge Avg. 
& Feature Avg. 
& Avg. Rank 
& $p$-value 
& W/T/L \\
\midrule
GCN     
& 46.82 & 47.23 & 45.98 & 11.52 & $1.84{\times}10^{-23}$ & 54/0/0 \\
GIN     
& 46.25 & 46.87 & 44.88 & 12.34 & $1.23{\times}10^{-23}$ & 54/0/0 \\
GMT     
& 48.03 & 48.46 & 47.18 & 7.78  & $4.25{\times}10^{-22}$ & 54/0/0 \\
CIN     
& 47.50 & 47.89 & 46.35 & 9.91  & $8.96{\times}10^{-23}$ & 54/0/0 \\
PathNN  
& 47.08 & 47.52 & 45.42 & 10.13 & $4.90{\times}10^{-23}$ & 54/0/0 \\
\midrule
DEAL    
& 57.10 & 56.87 & 49.67 & 3.75  & $3.38{\times}10^{-15}$ & 52/1/1 \\
SGDA    
& 48.30 & 49.45 & 49.20 & 6.88  & $4.71{\times}10^{-21}$ & 53/0/1 \\
A2GNN   
& 48.27 & 48.77 & 53.48 & 6.88  & $2.62{\times}10^{-19}$ & 54/0/0 \\
StruRW  
& 52.08 & 52.85 & 48.82 & 5.86  & $6.12{\times}10^{-16}$ & 53/0/1 \\
PA-BOTH 
& 52.63 & 52.54 & 48.83 & 5.69  & $1.37{\times}10^{-15}$ & 53/0/1 \\
GAA     
& 63.15 & 62.20 & 51.58 & 2.49  & $2.58{\times}10^{-8}$  & 48/1/5 \\
TDSS    
& 48.60 & 48.96 & 51.53 & 6.55  & $4.54{\times}10^{-19}$ & 52/0/2 \\
\midrule
\method{} 
& \textbf{66.89} 
& \textbf{66.45} 
& \textbf{58.92} 
& \textbf{1.22} 
& -- 
& -- \\
\bottomrule
\end{tabular}%
}
\vspace{-0.5cm}
\end{table}

\vspace{-0.2cm}
\subsection{Performance Comparison}

Table~\ref{tab:main_mutag_updated}, \ref{tab:mutag_node}-\ref{tab:nci1_idx} report the performance comparison under different domain shifts. From these tables, we can make the following observations: (1) Standard GNNs such as GCN, GIN, and GMT generally perform poorly under domain shifts, indicating that directly transferring a predictor trained on the retained source subset is insufficient when target graphs exhibit shifted structures or features. Since these methods do not exploit unlabeled target observations, their predictions are sensitive to reduced and potentially biased source evidence. (2) Existing GDA methods, including DEAL, GAA, and TDSS, usually improve over standard graph classifiers, confirming the benefit of adaptation objectives under source-target shifts. By leveraging representation alignment, propagation calibration, structural reweighting, target-side smoothing, or attribute-based adaptation, they can partially reduce domain discrepancy. However, their performance remains unstable across shift types and transfer directions, suggesting that existing GDA objectives still depend on reliable source semantics, which become fragile when only limited source graphs are retained. (3) \method{} achieves the best average performance and the lowest average rank, demonstrating its effectiveness for data-efficient agentic GDA. The gains mainly come from reliability-aware source anchors, which combine empirical prototypes with classifier directions to stabilize reusable source guidance under limited retained evidence. Moreover, prototype-aware soft target association aligns confidence-weighted target centers with source anchors, enabling unlabeled target observations to be associated with source semantics without hard pseudo-labels. Source-prior regularization further calibrates the target marginal by encouraging confident predictions while preventing drift toward unreliable classes. More results on other datasets can be found in Appendix~\ref{appendix:results}.

To summarize performance across all shift types, we report aggregate results in Table~\ref{tab:aggregate_summary_shift}. Node Avg., Edge Avg., and Feature Avg. denote the mean accuracy over node-density, edge-density, and feature-shift tasks, respectively. Avg. Rank is computed over all $54$ reported settings, where a smaller value indicates better overall performance. The $p$-value is obtained by a paired task-level $t$-test between \method{} and each baseline, and W/T/L counts the number of wins, ties, and losses according to the reported mean $\pm$ standard-deviation intervals. As shown in Table~\ref{tab:aggregate_summary_shift}, \method{} achieves the best average performance across all three shift categories and obtains the lowest Avg. Rank of $1.22$. Compared with the strongest baseline GAA, \method{} still achieves a clear advantage, with $48$ wins over $54$ settings. These results indicate that the improvement of \method{} is consistent across different shift types rather than being driven by a few favorable tasks, suggesting higher information yield from the same retained source budget.

\vspace{-0.2cm}
\subsection{Information Yield under Source-Data Budgets}

To evaluate how effectively each method exploits retained source evidence, we vary the source-data ratio $\rho$ within $\{0.1,0.3,0.5,0.7,0.9\}$ on Mutagenicity node shift and report the average accuracy in Figure~\ref{fig:efficiency_sensitivity}(a). We compare \method{} with two strong GDA baselines, GAA and DEAL, as well as the average performance of all compared baselines.

As shown in Figure~\ref{fig:efficiency_sensitivity}(a), \method{} outperforms both strong baselines and the average baseline across all source ratios. The advantage is particularly evident when the retained source ratio is small, where existing GDA methods suffer from weaker source supervision and less reliable alignment signals. By estimating reliability-aware anchors from the retained subset and using them to guide soft target association, \method{} extracts more useful adaptation signals from limited source evidence. As the source ratio increases, all methods benefit from additional source graphs, but \method{} maintains a clear margin, demonstrating that the proposed framework improves adaptation effectiveness under the same source-data budget.

\begin{figure*}[h]
    \centering
    \subfloat[Source data ratio $\rho$]{
        \includegraphics[width=0.32\textwidth]{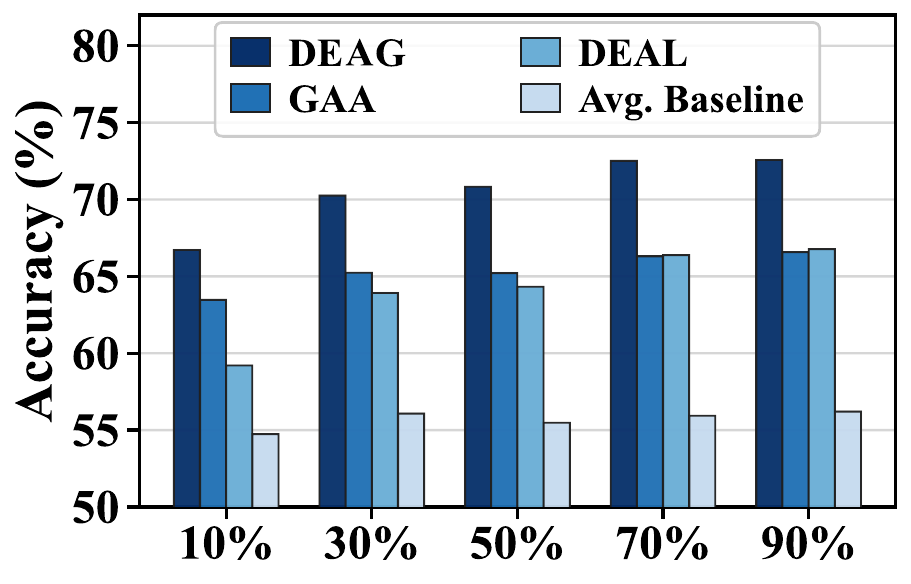}
        \label{fig:node_more}
    }
    \subfloat[$\lambda_1$]{
        \includegraphics[width=0.32\textwidth]{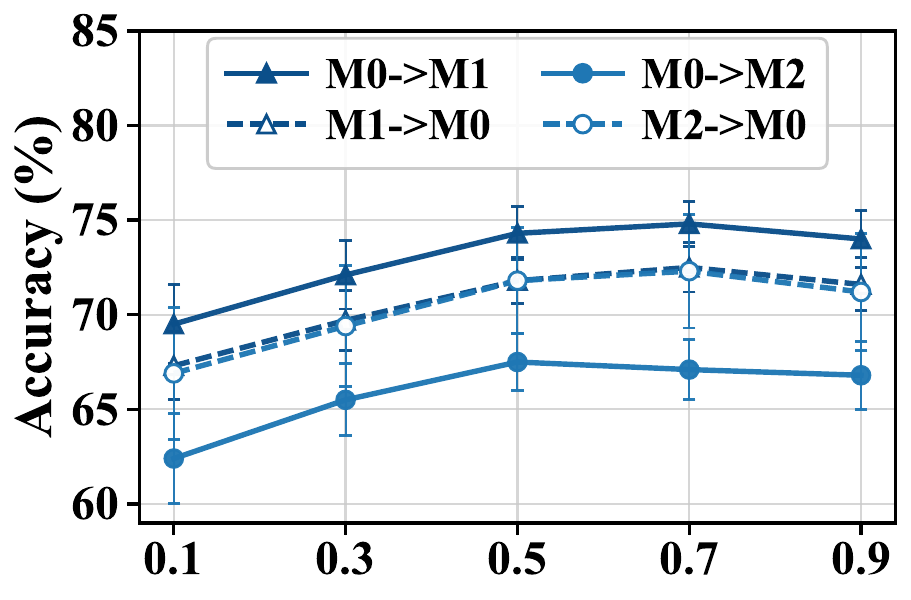}
        \label{fig:lambda1_mutag}
    }
    \subfloat[$\lambda_2$]{
        \includegraphics[width=0.32\textwidth]{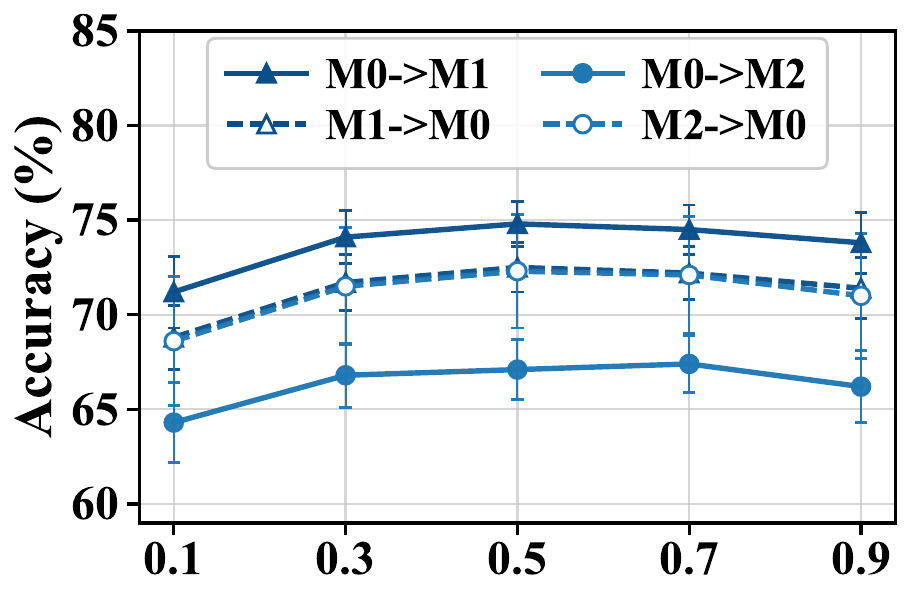}
        \label{fig:lambda2_mutag}
    }

    \caption{Information yield under source-data budgets and sensitivity analysis on Mutagenicity.}
    \label{fig:efficiency_sensitivity}
    \vspace{-0.5cm}
\end{figure*}

\subsection{Sensitivity Analysis}

We analyze the sensitivity of \method{} to the two balancing coefficients $\lambda_1$ and $\lambda_2$. The coefficient $\lambda_1$ controls the prototype alignment loss $\mathcal{L}_{pa}$, while $\lambda_2$ controls the source-prior target regularization loss $\mathcal{L}_{ptr}$. We vary each coefficient within $\{0.1,0.3,0.5,0.7,0.9\}$ and keep the other hyperparameters fixed.

Figures~\ref{fig:efficiency_sensitivity}(b)(c) show that \method{} is stable across a broad range of both coefficients. A small $\lambda_1$ provides insufficient guidance from reliability-aware source anchors, while a moderate value strengthens class-level semantic transfer. An overly large $\lambda_1$ can slightly reduce flexibility by forcing uncertain target centers toward imperfect anchors. For $\lambda_2$, a small value weakens target-marginal calibration and makes predictions more prone to drifting toward initially favored classes. A moderate value improves robustness by jointly sharpening target predictions and matching the target marginal to the retained source prior, whereas an excessively large value may over-constrain the target distribution. The default setting $\lambda_1=0.7$ and $\lambda_2=0.5$ provides a good balance between source-semantic alignment and target-marginal calibration.

\subsection{Ablation Study}

We conduct ablation studies to examine the contribution of each component in \method{}.
\method{} {w/o $\mathcal{L}_{pa}$} removes the prototype alignment loss, so target centers are no longer explicitly aligned with reliability-aware source anchors.
\method{} {w/o $\mathcal{L}_{ptr}$} removes the source-prior target regularization loss.
To further disentangle $\mathcal{L}_{ptr}$, \method{} {w/o $H_{\mathrm{cond}}$} removes the conditional entropy term, while \method{} {w/o $D_{\mathrm{prior}}$} removes the source-prior matching term.

\begin{table}[t]
\centering
\caption{The results of ablation studies on the Mutagenicity dataset. \textbf{Bold} results indicate the best performance.}
\label{tab:ablation_mutag}
\vspace{0.2cm}
\resizebox{\textwidth}{!}{%
\begin{tabular}{lcccccccccccc}
\toprule
Method & M0$\rightarrow$M1 & M0$\rightarrow$M2 & M0$\rightarrow$M3 & M1$\rightarrow$M0 & M1$\rightarrow$M2 & M1$\rightarrow$M3 & M2$\rightarrow$M0 & M2$\rightarrow$M1 & M2$\rightarrow$M3 & M3$\rightarrow$M0 & M3$\rightarrow$M1 & M3$\rightarrow$M2 \\
\midrule
\method{} w/o $\mathcal{L}_{pa}$       & 70.8 & 63.4 & 57.1 & 65.4 & 67.5 & 59.8 & 62.3 & 73.1 & 61.2 & 58.6 & 60.1 & 67.8 \\
\method{} w/o $\mathcal{L}_{ptr}$      & 73.4 & 66.1 & 60.3 & 67.3 & 70.2 & 62.5 & 65.0 & 75.8 & 64.1 & 62.1 & 63.4 & 70.6 \\
\method{} w/o $H_{\text{cond}}$        & 74.9 & 67.5 & 61.6 & 69.8 & 71.8 & 64.1 & 66.9 & 77.4 & 65.7 & 63.4 & 65.0 & 72.3 \\
\method{} w/o $D_{\text{prior}}$       & 75.3 & 68.2 & 61.2 & 70.7 & 72.3 & 64.6 & 67.8 & 77.8 & 66.2 & 63.1 & 65.5 & 72.7 \\
\midrule
\method{}                  & \textbf{76.9} & \textbf{69.7} & \textbf{63.4} & \textbf{71.8} & \textbf{73.7} & \textbf{65.9} & \textbf{69.1} & \textbf{79.2} & \textbf{67.5} & \textbf{65.0} & \textbf{66.8} & \textbf{74.0} \\
\bottomrule
\end{tabular}%
}
\vspace{-0.4cm}
\end{table}

Table~\ref{tab:ablation_mutag} reports the ablation results. From the table, we make four observations: (1) The full model achieves the best performance across the ablation settings, indicating that the three components of \method{} provide complementary benefits rather than redundant regularization.
(2) Removing $\mathcal{L}_{pa}$ leads to the most consistent and pronounced degradation. This trend shows that reliability-aware prototype alignment is the primary mechanism for transferring reusable class-level source semantics to the target domain. Without this loss, target representations are no longer explicitly guided toward source anchors, making adaptation more vulnerable to weakened source evidence.
(3) Removing $\mathcal{L}_{ptr}$ also causes clear performance drops across transfer directions, suggesting that prototype alignment alone cannot fully stabilize target predictions. This confirms the necessity of source-prior target regularization, especially when early target associations may be biased under limited source evidence.
(4) Removing either $H_{\mathrm{cond}}$ or $D_{\mathrm{prior}}$ weakens the model, but the degradation is generally milder than removing the entire $\mathcal{L}_{ptr}$. This indicates that the two terms play complementary roles: $H_{\mathrm{cond}}$ improves sample-level prediction confidence, while $D_{\mathrm{prior}}$ constrains the target marginal using retained source evidence.
Overall, the ablation trends verify that \method{} benefits from the joint effect of reliability-aware source anchoring, prototype-aware target association, and source-prior target-marginal calibration in the limited-source setting.
\vspace{-0.2cm}
\section{Conclusion}

We propose \method{}, a reliability-aware prototype learning framework for data-efficient agentic graph domain adaptation. Using source support and embedding compactness, \method{} blends empirical prototypes with classifier directions into reusable anchors. Soft association aligns confidence-weighted target centers with these anchors without hard pseudo-label training, while source-prior marginal calibration stabilizes predictions using retained source evidence. Experiments across structural and feature shifts demonstrate average gains over competitive baselines under matched source-data budgets.

\bibliographystyle{splncs04}
\bibliography{reference}

@String { ACMMM        = {Proceedings of the ACM International Conference on Multimedia} }

@String { JMLR         = {The Journal of Machine Learning Research.} }

@String { ICLR         = {Proceedings of the International Conference on Learning Representations
} }

@String { ICML         = {Proceedings of the International Conference on Machine Learning} }

@String { NIPS         = {Proceedings of the Conference on Neural Information Processing Systems} }

@String { TPAMI         = {IEEE Transactions on Pattern Analysis and Machine Intelligence} }

@String { AAAI         = {Proceedings of the AAAI Conference on Artificial Intelligence}}

@String { WWW        = {Proceedings of the ACM Web Conference}}

@String { TKDD        = {ACM Transactions on Knowledge Discovery from Data}}

@String { TKDE        = {IEEE Transactions on Knowledge and Data Engineering}}

@String { WSDM       = {Proceedings of the International Conference on Web Search and Data Mining}}

@String { WSDM       = {Proceedings of the International ACM Conference on Web Search \& Data Mining}}

@String { KDD        = {Proceedings of the International ACM SIGKDD Conference on Knowledge Discovery \& Data Mining}}

@article{kipf2016semi,
  title={Semi-supervised classification with graph convolutional networks},
  author={Kipf, Thomas N and Welling, Max},
  journal={arXiv preprint arXiv:1609.02907},
  year={2016}
}

@article{xu2018powerful,
  title={How powerful are graph neural networks?},
  author={Xu, Keyulu and Hu, Weihua and Leskovec, Jure and Jegelka, Stefanie},
  journal={arXiv preprint arXiv:1810.00826},
  year={2018}
}

@inproceedings{yin2022deal,
  title={Deal: An unsupervised domain adaptive framework for graph-level classification},
  author={Yin, Nan and Shen, Li and Li, Baopu and Wang, Mengzhu and Luo, Xiao and Chen, Chong and Luo, Zhigang and Hua, Xian-Sheng},
  booktitle=ACMMM,
  year={2022}
}

@inproceedings{yin2023coco,
  title={Coco: A coupled contrastive framework for unsupervised domain adaptive graph classification},
  author={Yin, Nan and Shen, Li and Wang, Mengzhu and Lan, Long and Ma, Zeyu and Chen, Chong and Hua, Xian-Sheng and Luo, Xiao},
  booktitle=ICML,
  year={2023}
}

@article{yin2025dream,
  title={Dream: a dual variational framework for unsupervised graph domain adaptation},
  author={Yin, Nan and Shen, Li and Wang, Mengzhu and Liu, Xinwang and Chen, Chong and Hua, Xian-Sheng},
  journal=TPAMI,
  year={2025},
  publisher={IEEE}
}

@inproceedings{liu2023structural,
  title={Structural re-weighting improves graph domain adaptation},
  author={Liu, Shikun and Li, Tianchun and Feng, Yongbin and Tran, Nhan and Zhao, Han and Qiu, Qiang and Li, Pan},
  booktitle=ICML,
  year={2023},
  organization={PMLR}
}

@inproceedings{liu2024rethinking,
  title={Rethinking propagation for unsupervised graph domain adaptation},
  author={Liu, Meihan and Fang, Zeyu and Zhang, Zhen and Gu, Ming and Zhou, Sheng and Wang, Xin and Bu, Jiajun},
  booktitle=AAAI,
  year={2024}
}

@article{liu2024pairwise,
  title={Pairwise alignment improves graph domain adaptation},
  author={Liu, Shikun and Zou, Deyu and Zhao, Han and Li, Pan},
  journal={arXiv preprint arXiv:2403.01092},
  year={2024}
}

@inproceedings{chen2025smoothness,
  title={Smoothness really matters: A simple yet effective approach for unsupervised graph domain adaptation},
  author={Chen, Wei and Ye, Guo and Wang, Yakun and Zhang, Zhao and Zhang, Libang and Wang, Daixin and Zhang, Zhiqiang and Zhuang, Fuzhen},
  booktitle=AAAI,
  year={2025}
}

@inproceedings{wang2026nested,
  title={Nested graph pseudo-label refinement for noisy label domain adaptation learning},
  author={Wang, Yingxu and Wang, Mengzhu and Huang, Zhichao and Liu, Suyu and Yin, Nan},
  booktitle=AAAI,
  year={2026}
}

@article{ganin2016domain,
  title={Domain-adversarial training of neural networks},
  author={Ganin, Yaroslav and Ustinova, Evgeniya and Ajakan, Hana and Germain, Pascal and Larochelle, Hugo and Laviolette, Fran{\c{c}}ois and March, Mario and Lempitsky, Victor},
  journal=JMLR,
  year={2016}
}

@inproceedings{long2015learning,
  title={Learning transferable features with deep adaptation networks},
  author={Long, Mingsheng and Cao, Yue and Wang, Jianmin and Jordan, Michael},
  booktitle=ICML,
  year={2015},
  organization={PMLR}
}

@inproceedings{liang2020we,
  title={Do we really need to access the source data? source hypothesis transfer for unsupervised domain adaptation},
  author={Liang, Jian and Hu, Dapeng and Feng, Jiashi},
  booktitle=ICML,
  year={2020},
  organization={PMLR}
}

@inproceedings{garg2023rlsbench,
  title={Rlsbench: Domain adaptation under relaxed label shift},
  author={Garg, Saurabh and Erickson, Nick and Sharpnack, James and Smola, Alex and Balakrishnan, Sivaraman and Lipton, Zachary Chase},
  booktitle=ICML,
  year={2023},
  organization={PMLR}
}

@article{zhang2026brainriem,
  title={BrainRiem: Riemannian Prototype Learning for Source-Free Cross-Site Brain Network Diagnosis},
  author={Zhang, Kunyu and Xu, Tianxiang},
  journal={arXiv preprint arXiv:2606.29200},
  year={2026}
}

@article{ju2024survey,
  title={A survey of data-efficient graph learning},
  author={Ju, Wei and Yi, Siyu and Wang, Yifan and Long, Qingqing and Luo, Junyu and Xiao, Zhiping and Zhang, Ming},
  journal={arXiv preprint arXiv:2402.00447},
  year={2024}
}

@article{hashemi2024comprehensive,
  title={A comprehensive survey on graph reduction: Sparsification, coarsening, and condensation},
  author={Hashemi, Mohammad and Gong, Shengbo and Ni, Juntong and Fan, Wenqi and Prakash, B Aditya and Jin, Wei},
  journal={arXiv preprint arXiv:2402.03358},
  year={2024}
}

@article{zhang2024navigating,
  title={Navigating complexity: Toward lossless graph condensation via expanding window matching},
  author={Zhang, Yuchen and Zhang, Tianle and Wang, Kai and Guo, Ziyao and Liang, Yuxuan and Bresson, Xavier and Jin, Wei and You, Yang},
  journal={arXiv preprint arXiv:2402.05011},
  year={2024}
}

@article{bodnar2021weisfeiler,
  title={Weisfeiler and lehman go cellular: Cw networks},
  author={Bodnar, Cristian and Frasca, Fabrizio and Otter, Nina and Wang, Yuguang and Lio, Pietro and Montufar, Guido F and Bronstein, Michael},
  journal=NIPS,
  year={2021}
}

@article{baek2021accurate,
  title={Accurate learning of graph representations with graph multiset pooling},
  author={Baek, Jinheon and Kang, Minki and Hwang, Sung Ju},
  journal={arXiv preprint arXiv:2102.11533},
  year={2021}
}

@article{kazius2005derivation,
  title={Derivation and validation of toxicophores for mutagenicity prediction},
  author={Kazius, Jeroen and McGuire, Ross and Bursi, Roberta},
  journal={Journal of medicinal chemistry},
  year={2005},
  publisher={ACS Publications}
}

@article{wale2008comparison,
  title={Comparison of descriptor spaces for chemical compound retrieval and classification},
  author={Wale, Nikil and Watson, Ian A and Karypis, George},
  journal={Knowledge and Information Systems},
  year={2008},
  publisher={Springer}
}

@inproceedings{michel2023path,
  title={Path neural networks: Expressive and accurate graph neural networks},
  author={Michel, Gaspard and Nikolentzos, Giannis and Lutzeyer, Johannes F and Vazirgiannis, Michalis},
  booktitle=ICML,
  year={2023},
  organization={PMLR}
}

@article{qiao2023semi,
  title={Semi-supervised domain adaptation in graph transfer learning},
  author={Qiao, Ziyue and Luo, Xiao and Xiao, Meng and Dong, Hao and Zhou, Yuanchun and Xiong, Hui},
  journal={arXiv preprint arXiv:2309.10773},
  year={2023}
}

@inproceedings{fang2025benefits,
  title={On the benefits of attribute-driven graph domain adaptation},
  author={Fang, Ruiyi and Li, Bingheng and Zeng, Qiuhao and Hosseini Dashtbayaz, Nima and Pu, Ruizhi and Ling, Charles and Wang, Boyu and others},
  booktitle=ICLR,
  year={2025}
}

@article{wang2026usbd,
  title={Usbd: Universal structural basis distillation for source-free graph domain adaptation},
  author={Wang, Yingxu and Zhang, Kunyu and Wang, Mengzhu and Gao, Siyang and Yin, Nan},
  journal={arXiv preprint arXiv:2602.08431},
  year={2026}
}

@article{wang2024degree,
  title={Degree-Conscious Spiking Graph for Cross-Domain Adaptation},
  author={Wang, Yingxu and Wang, Mengzhu and Su, Houcheng and Yin, Nan and Yao, Quanming and Kwok, James},
  journal={arXiv preprint arXiv:2410.06883},
  year={2024}
}

@article{wang2026cross,
  title={Cross-Resolution Semantic Learning for Graph Domain Adaptation},
  author={Wang, Yingxu and Huang, Haoze and Zheng, Zhongkai and Liang, Shangsong},
  journal={arXiv preprint arXiv:2607.29365},
  year={2026}
}

@article{wang2026disrfm,
  title={DisRFM: Polar Riemannian Flow Matching for Structure-Preserving Graph Domain Adaptation},
  author={Wang, Yingxu and Liu, Xinwang and Wang, Mengzhu and Gao, Siyang and Yin, Nan},
  journal={arXiv preprint arXiv:2602.00656},
  year={2026}
}

@article{gao2025graph,
  title={Graph condensation: A survey},
  author={Gao, Xinyi and Yu, Junliang and Chen, Tong and Ye, Guanhua and Zhang, Wentao and Yin, Hongzhi},
  journal=TKDE,
  year={2025},
  publisher={IEEE}
}

@article{tian2025knowledge,
  title={Knowledge distillation on graphs: A survey},
  author={Tian, Yijun and Pei, Shichao and Zhang, Xiangliang and Zhang, Chuxu and Chawla, Nitesh V},
  journal={ACM Computing Surveys},
  year={2025},
  publisher={ACM New York, NY}
}

@inproceedings{hu2025large,
  title={Large language model meets graph neural network in knowledge distillation},
  author={Hu, Shengxiang and Zou, Guobing and Yang, Song and Lin, Shiyi and Gan, Yanglan and Zhang, Bofeng and Chen, Yixin},
  booktitle=AAAI,
  year={2025}
}

@inproceedings{yu2025samgpt,
  title={Samgpt: Text-free graph foundation model for multi-domain pre-training and cross-domain adaptation},
  author={Yu, Xingtong and Gong, Zechuan and Zhou, Chang and Fang, Yuan and Zhang, Hui},
  booktitle=WWW,
  year={2025}
}

@inproceedings{ngo2025higda,
  title={Higda: Hierarchical graph of nodes to learn local-to-global topology for semi-supervised domain adaptation},
  author={Ngo, Ba Hung and Bui, Doanh C and Do-Tran, Nhat-Tuong and Choi, Tae Jong},
  booktitle=AAAI,
  year={2025}
}

@inproceedings{you2023graph,
  title={Graph domain adaptation via theory-grounded spectral regularization},
  author={You, Yuning and Chen, Tianlong and Wang, Zhangyang and Shen, Yang},
  booktitle=ICLR,
  year={2023}
}

@article{cai2024graph,
  title={Graph domain adaptation: A generative view},
  author={Cai, Ruichu and Wu, Fengzhu and Li, Zijian and Wei, Pengfei and Yi, Lingling and Zhang, Kun},
  journal=TKDD,
  year={2024}
}

@article{dai2022graph,
  title={Graph transfer learning via adversarial domain adaptation with graph convolution},
  author={Dai, Quanyu and Wu, Xiao-Ming and Xiao, Jiaren and Shen, Xiao and Wang, Dan},
  journal=TKDE,
  year={2022},
  publisher={IEEE}
}

@article{zhang2024anchor,
  title={Anchor guided unsupervised domain adaptation},
  author={Zhang, Canyu and Nie, Feiping and Wang, Rong},
  journal=TKDE,
  year={2024},
  publisher={IEEE}
}

@article{zhang2026survey,
  title={A survey of deep graph learning under distribution shifts: from graph out-of-distribution generalization to adaptation},
  author={Zhang, Kexin and Liu, Shuhan and Wang, Song and Shi, Weili and Chen, Chen and Li, Pan and Li, Sheng and Li, Jundong and Ding, Kaize},
  journal=TKDD,
  year={2026}
}

@article{li2023sigma,
  title={Sigma++: Improved semantic-complete graph matching for domain adaptive object detection},
  author={Li, Wuyang and Liu, Xinyu and Yuan, Yixuan},
  journal=TPAMI,
  volume={45},
  number={7},
  pages={9022--9040},
  year={2023},
  publisher={IEEE}
}

@article{fang2025homophily,
  title={Homophily enhanced graph domain adaptation},
  author={Fang, Ruiyi and Li, Bingheng and Zhao, Jingyu and Pu, Ruizhi and Zeng, Qiuhao and Xu, Gezheng and Ling, Charles and Wang, Boyu},
  journal={arXiv preprint arXiv:2505.20089},
  year={2025}
}

@article{lei2025gradual,
  title={Gradual domain adaptation for graph learning},
  author={Lei, Pui Ieng and Chen, Ximing and Sheng, Yijun and Liu, Yanyan and Gong, Zhiguo and Yang, Qiang},
  journal={ACM Transactions on Intelligent Systems and Technology},
  year={2025},
  publisher={ACM New York, NY}
}

@article{zeng2025pave,
  title={Pave your own path: Graph gradual domain adaptation on fused gromov-wasserstein geodesics},
  author={Zeng, Zhichen and Qiu, Ruizhong and Bao, Wenxuan and Wei, Tianxin and Lin, Xiao and Yan, Yuchen and Abdelzaher, Tarek F and Han, Jiawei and Tong, Hanghang},
  journal={arXiv preprint arXiv:2505.12709},
  year={2025}
}

@inproceedings{wang2025correction,
  title={Correction to: Domain Graph-Structured Multi-source Domain Adaptation with Dual Integration},
  author={Wang, Jiayi and Zheng, Xin and Li, Yi and Guo, Yanqing},
  booktitle={International Conference on Advanced Data Mining and Applications},
  year={2025},
  organization={Springer}
}

@article{hadipour2025graphban,
  title={GraphBAN: An inductive graph-based approach for enhanced prediction of compound-protein interactions},
  author={Hadipour, Hamid and Li, Yan Yi and Sun, Yan and Deng, Chutong and Lac, Leann and Davis, Rebecca and Cardona, Silvia T and Hu, Pingzhao},
  journal={Nature Communications},
  year={2025},
  publisher={Nature Publishing Group UK London}
}

@inproceedings{yang2025graphlora,
  title={Graphlora: Structure-aware contrastive low-rank adaptation for cross-graph transfer learning},
  author={Yang, Zhe-Rui and Han, Jindong and Wang, Chang-Dong and Liu, Hao},
  booktitle=KDD,
  year={2025}
}

@inproceedings{wang2025bridging,
  title={Bridging source and target domains via link prediction for unsupervised domain adaptation on graphs},
  author={Wang, Yilong and Zhao, Tianxiang and Wu, Zongyu and Wang, Suhang},
  booktitle=WSDM,
  year={2025}
}

@inproceedings{hu2025automated,
  title={Automated design of agentic systems},
  author={Hu, Shengran and Lu, Cong and Clune, Jeff},
  booktitle=ICLR,
  year={2025}
}

@article{collaco2026role,
  title={The role of agentic artificial intelligence in healthcare: a scoping review},
  author={Collaco, Bernardo G and Haider, Syed Ali and Prabha, Srinivasagam and Gomez-Cabello, Cesar A and Genovese, Ariana and Wood, Nadia G and Bagaria, Sanjay P and Gopala, Narayanan and Tao, Cui and Forte, Antonio Jorge},
  journal={npj Digital Medicine},
  year={2026}
}

@article{wang2025protomol,
  title={Protomol: enhancing molecular property prediction via prototype-guided multimodal learning},
  author={Wang, Yingxu and Zhang, Kunyu and Huang, Jiaxin and Yin, Nan and Liu, Siwei and Segal, Eran},
  journal={Briefings in Bioinformatics},
  volume={26},
  number={6},
  pages={bbaf629},
  year={2025},
  publisher={Oxford University Press}
}

@article{wang2026sgac,
  title={SGAC: a graph neural network framework for imbalanced and structure-aware AMP classification},
  author={Wang, Yingxu and Liang, Victor and Yin, Nan and Liu, Siwei and Segal, Eran},
  journal={Briefings in Bioinformatics},
  volume={27},
  number={1},
  pages={bbag038},
  year={2026},
  publisher={Oxford University Press}
}

@article{wang2025dusego,
  title={Dusego: Dual second-order equivariant graph ordinary differential equation},
  author={Wang, Yingxu and Yin, Nan and Xiao, Mingyan and Yi, Xinhao and Liu, Siwei and Liang, Shangsong},
  journal={ACM Transactions on Knowledge Discovery from Data},
  volume={20},
  number={1},
  pages={1--18},
  year={2025},
  publisher={ACM New York, NY}
}

\appendix
\newpage
\section{Proof of Theorem~\ref{thm:generalization}}
\label{app:proof_generalization}

\textbf{Theorem \ref{thm:generalization} (Target-Risk Bound)}\textit{
Let $f$ be the graph encoder, $h$ be the classifier, and $g=h\circ f\in\mathcal{H}$ be the learned predictor. Let $\mathcal{R}_{T}(g)$ denote the expected target risk and $\widehat{\mathcal{R}}_{\mathcal{S}_S^{\rho}}(g)$ denote the empirical risk on the retained source subset $\mathcal{S}_S^{\rho}$ with $K$ samples. For each class $c$, let $\mathbf{a}_c$ be the learned reliability-aware source anchor and $\mathbf{a}_c^{\star}$ be the ideal class anchor estimated with sufficient source support. Define the retained-source prior $\pi_{S,c}=K_c/K$, the anchor estimation error $\epsilon_{\mathrm{anc}}=\sum_{c=1}^{C}\pi_{S,c}\|\mathbf{a}_c-\mathbf{a}_c^{\star}\|_2$, and the class-mass factor $\gamma_{\beta}=\sum_{c=1}^{C}\pi_{S,c}^{2}/\beta_c$, where $\beta_c>0$ is the detached target class-mass weight in $\mathcal{L}_{pa}$. Assume the loss $\ell$ is bounded in $[0,1]$ and $K_{\mathrm{Lip}}$-Lipschitz with respect to normalized graph representations. Then, with probability at least $1-\delta$, the target risk is bounded by
\begin{equation}
\label{eq:generalization_bound}
\begin{aligned}
\mathcal{R}_{T}(g)
\leq\;&
\widehat{\mathcal{R}}_{\mathcal{S}_S^{\rho}}(g)
+
K_{\mathrm{Lip}}\sqrt{\gamma_{\beta}\mathcal{L}_{pa}}
+
\sqrt{2\mathcal{D}_{prior}}
+
K_{\mathrm{Lip}}\epsilon_{\mathrm{anc}}
\\
&+
2\mathfrak{R}_{K}(\mathcal{H})
+
3\sqrt{\frac{\log(2/\delta)}{2K}}
+
\lambda^{\star},
\end{aligned}
\end{equation}
where $\mathfrak{R}_{K}(\mathcal{H})$ is the empirical Rademacher complexity of $\mathcal{H}$ under $K$ retained source samples, and $\lambda^{\star}$ denotes the minimal joint risk and residual mismatch between the retained source and target domains.}

\begin{proof}
Let $g=h\circ f$ denote the learned predictor. We use $\mathcal{P}_{S}^{\rho}$ to denote the retained-source distribution induced by $\mathcal{S}_{S}^{\rho}$, and $\mathcal{P}_{T}$ to denote the target distribution. The expected risks on the retained-source and target domains are denoted by $\mathcal{R}_{S}^{\rho}(g)$ and $\mathcal{R}_{T}(g)$, respectively.

Since the loss $\ell$ is bounded in $[0,1]$, the standard Rademacher generalization inequality implies that, with probability at least $1-\delta$,

\begin{equation}
\label{eq:appendix_source_generalization}
\mathcal{R}_{S}^{\rho}(g)
\leq
\widehat{\mathcal{R}}_{\mathcal{S}_{S}^{\rho}}(g)
+
2\mathfrak{R}_{K}(\mathcal{H})
+
3\sqrt{
\frac{\log(2/\delta)}{2K}
},
\end{equation}

where $\mathfrak{R}_{K}(\mathcal{H})$ is the empirical Rademacher complexity of $\mathcal{H}$ under $K$ retained source samples.

We next relate the retained-source risk to the target risk. Following the standard domain adaptation decomposition, the target risk can be upper bounded by the retained-source risk, a class-conditional representation discrepancy, a marginal-prior discrepancy, and an irreducible residual term:

\begin{equation}
\label{eq:appendix_da_decomposition}
\mathcal{R}_{T}(g)
\leq
\mathcal{R}_{S}^{\rho}(g)
+
\Delta_{\mathrm{cls}}(g)
+
\Delta_{\mathrm{prior}}(g)
+
\lambda^{\star},
\end{equation}

where $\lambda^{\star}$ denotes the minimal joint risk and residual mismatch between the retained-source and target domains.

We first bound the class-conditional discrepancy $\Delta_{\mathrm{cls}}(g)$. For each class $c$, let $\boldsymbol{\mu}_{T,c}$ denote the confidence-weighted target center, $\mathbf{a}_c$ denote the learned reliability-aware source anchor, and $\mathbf{a}_c^{\star}$ denote the ideal source anchor obtained with sufficient source support. Since the loss is $K_{\mathrm{Lip}}$-Lipschitz with respect to normalized graph representations, the discrepancy between the target class center and the ideal source class anchor satisfies

\begin{equation}
\label{eq:appendix_cls_lip}
\Delta_{\mathrm{cls}}(g)
\leq
K_{\mathrm{Lip}}
\sum_{c=1}^{C}
\pi_{S,c}
\left\|
\boldsymbol{\mu}_{T,c}
-
\mathbf{a}_c^{\star}
\right\|_2.
\end{equation}

By the triangle inequality, for every class $c$,

\begin{equation}
\label{eq:appendix_triangle}
\left\|
\boldsymbol{\mu}_{T,c}
-
\mathbf{a}_c^{\star}
\right\|_2
\leq
\left\|
\boldsymbol{\mu}_{T,c}
-
\mathbf{a}_c
\right\|_2
+
\left\|
\mathbf{a}_c
-
\mathbf{a}_c^{\star}
\right\|_2.
\end{equation}

Substituting Eq.~\eqref{eq:appendix_triangle} into Eq.~\eqref{eq:appendix_cls_lip} yields

\begin{equation}
\label{eq:appendix_cls_split}
\begin{aligned}
\Delta_{\mathrm{cls}}(g)
\leq\;&
K_{\mathrm{Lip}}
\sum_{c=1}^{C}
\pi_{S,c}
\left\|
\boldsymbol{\mu}_{T,c}
-
\mathbf{a}_c
\right\|_2
\\
&+
K_{\mathrm{Lip}}
\sum_{c=1}^{C}
\pi_{S,c}
\left\|
\mathbf{a}_c
-
\mathbf{a}_c^{\star}
\right\|_2.
\end{aligned}
\end{equation}

By the definition of the weighted anchor estimation error,

\begin{equation}
\label{eq:appendix_anchor_error}
\epsilon_{\mathrm{anc}}
=
\sum_{c=1}^{C}
\pi_{S,c}
\left\|
\mathbf{a}_c
-
\mathbf{a}_c^{\star}
\right\|_2,
\end{equation}

we have

\begin{equation}
\label{eq:appendix_cls_with_anchor}
\Delta_{\mathrm{cls}}(g)
\leq
K_{\mathrm{Lip}}
\sum_{c=1}^{C}
\pi_{S,c}
\left\|
\boldsymbol{\mu}_{T,c}
-
\mathbf{a}_c
\right\|_2
+
K_{\mathrm{Lip}}
\epsilon_{\mathrm{anc}}.
\end{equation}

It remains to relate the first term in Eq.~\eqref{eq:appendix_cls_with_anchor} to the prototype alignment loss. Let

\begin{equation}
\label{eq:appendix_dc}
d_c
=
\left\|
\boldsymbol{\mu}_{T,c}
-
\mathbf{a}_c
\right\|_2.
\end{equation}

The prototype alignment loss is

\begin{equation}
\label{eq:appendix_lpa}
\mathcal{L}_{pa}
=
\sum_{c=1}^{C}
\beta_c d_c^2.
\end{equation}

Assuming $\beta_c>0$ for all classes involved in alignment, we apply the Cauchy--Schwarz inequality:

\begin{equation}
\label{eq:appendix_cs_derivation}
\begin{aligned}
\sum_{c=1}^{C}
\pi_{S,c}d_c
&=
\sum_{c=1}^{C}
\frac{\pi_{S,c}}{\sqrt{\beta_c}}
\sqrt{\beta_c}d_c
\\
&\leq
\left(
\sum_{c=1}^{C}
\frac{\pi_{S,c}^{2}}{\beta_c}
\right)^{1/2}
\left(
\sum_{c=1}^{C}
\beta_c d_c^2
\right)^{1/2}
\\
&=
\sqrt{
\gamma_{\beta}
\mathcal{L}_{pa}
},
\end{aligned}
\end{equation}

where

\begin{equation}
\label{eq:appendix_gamma}
\gamma_{\beta}
=
\sum_{c=1}^{C}
\frac{\pi_{S,c}^{2}}{\beta_c}.
\end{equation}

Combining Eq.~\eqref{eq:appendix_cls_with_anchor} and Eq.~\eqref{eq:appendix_cs_derivation}, we obtain

\begin{equation}
\label{eq:appendix_cls_final}
\Delta_{\mathrm{cls}}(g)
\leq
K_{\mathrm{Lip}}
\sqrt{
\gamma_{\beta}\mathcal{L}_{pa}
}
+
K_{\mathrm{Lip}}
\epsilon_{\mathrm{anc}}.
\end{equation}

We then bound the marginal-prior discrepancy. Recall that the retained source prior is $\pi_S$ and the target marginal prediction is $\bar{p}_T$. The prior matching term in \method{} is

\begin{equation}
\label{eq:appendix_dprior}
\mathcal{D}_{prior}
=
\mathrm{KL}
\left(
\pi_S
\Vert
\bar{p}_T
\right).
\end{equation}

By Pinsker's inequality,

\begin{equation}
\label{eq:appendix_pinsker}
\left\|
\pi_S
-
\bar{p}_T
\right\|_1
\leq
\sqrt{
2
\mathrm{KL}
\left(
\pi_S
\Vert
\bar{p}_T
\right)
}
=
\sqrt{
2\mathcal{D}_{prior}
}.
\end{equation}

Therefore, the prior discrepancy satisfies

\begin{equation}
\label{eq:appendix_prior_bound}
\Delta_{\mathrm{prior}}(g)
\leq
\sqrt{
2\mathcal{D}_{prior}
}.
\end{equation}

Substituting Eq.~\eqref{eq:appendix_source_generalization}, Eq.~\eqref{eq:appendix_cls_final}, and Eq.~\eqref{eq:appendix_prior_bound} into Eq.~\eqref{eq:appendix_da_decomposition}, we obtain

\begin{equation}
\label{eq:appendix_final_bound}
\begin{aligned}
\mathcal{R}_{T}(g)
\leq\;&
\widehat{\mathcal{R}}_{\mathcal{S}_{S}^{\rho}}(g)
+
K_{\mathrm{Lip}}
\sqrt{
\gamma_{\beta}
\mathcal{L}_{pa}
}
+
\sqrt{
2\mathcal{D}_{prior}
}
+
K_{\mathrm{Lip}}
\epsilon_{\mathrm{anc}}
\\
&+
2\mathfrak{R}_{K}(\mathcal{H})
+
3\sqrt{
\frac{\log(2/\delta)}{2K}
}
+
\lambda^{\star}.
\end{aligned}
\end{equation}
\end{proof}

\begin{table}[t] \centering \caption{Statistics of the experimental datasets.} \vspace{0.2cm} \setlength{\tabcolsep}{3pt} \begin{tabular}{lcccc} \toprule Datasets & Graphs & Avg. Nodes & Avg. Edges & Classes \\ \midrule NCI1 & 4,110 & 29.87 & 32.30 & 2 \\ Mutagenicity & 4,337 & 30.32 & 30.77 & 2 \\ \midrule DD & 1,178 & 284.32 & 715.66 & 2 \\ PROTEINS & 1,113 & 39.1 & 72.8 & 2 \\ COX2 & 467 & 41.22 & 43.45& 2 \\ COX2\_MD & 303 & 26.28 & 335.12 & 2\\ BZR & 405 & 35.75 & 38.36 & 2 \\ BZR\_MD & 306 & 21.30 & 225.06 & 2 \\ \bottomrule \end{tabular} \label{tab:dataset} \end{table}

\section{Datasets}\label{sec:dataset}

\section{More Results}\label{appendix:results}
\begin{table}[htbp]
\centering
\caption{The classification results (in \%) on the Mutagenicity dataset under node density domain shift (source → target). M0, M1, M2, and M3 denote the sub-datasets partitioned with node density. Bold results indicate the best performance.}
\vspace{0.2cm}
\resizebox{\textwidth}{!}{%
\begin{tabular}{lcccccccccccc}
\toprule
Method  & M0$\rightarrow$M1 & M0$\rightarrow$M2 & M0$\rightarrow$M3 & M1$\rightarrow$M0 & M1$\rightarrow$M2 & M1$\rightarrow$M3 & M2$\rightarrow$M0 & M2$\rightarrow$M1 & M2$\rightarrow$M3 & M3$\rightarrow$M0 & M3$\rightarrow$M1 & M3$\rightarrow$M2 \\
\midrule
GCN     & 53.2$\pm$4.1 & 49.5$\pm$3.6 & 50.1$\pm$4.5 & 46.5$\pm$2.4 & 58.2$\pm$2.1 & 45.7$\pm$1.6 & 46.1$\pm$1.8 & 65.3$\pm$2.5 & 45.1$\pm$1.9 & 48.6$\pm$2.8 & 41.2$\pm$6.2  & 43.5$\pm$5.1 \\
GIN     & 51.4$\pm$5.2 & 47.8$\pm$4.4 & 49.6$\pm$3.8 & 45.1$\pm$2.9 & 56.8$\pm$2.5 & 44.3$\pm$1.9 & 44.8$\pm$2.1 & 63.9$\pm$3.1 & 43.7$\pm$2.2 & 47.1$\pm$3.1 & 39.5$\pm$5.8  & 41.8$\pm$6.5 \\
GMT     & 55.1$\pm$3.8 & 51.2$\pm$3.2 & 51.5$\pm$2.9 & 47.8$\pm$2.1 & 59.4$\pm$1.8 & 46.8$\pm$1.4 & 47.3$\pm$1.5 & 66.8$\pm$2.2 & 46.3$\pm$1.6 & 49.5$\pm$2.4 & 42.6$\pm$5.5  & 44.7$\pm$4.8 \\
CIN     & 54.3$\pm$4.5 & 50.6$\pm$3.7 & 50.9$\pm$3.4 & 47.2$\pm$2.3 & 58.9$\pm$2.0 & 46.2$\pm$1.5 & 46.7$\pm$1.7 & 66.1$\pm$2.4 & 45.8$\pm$1.8 & 48.9$\pm$2.9 & 41.9$\pm$6.0  & 43.9$\pm$5.6 \\
PathNN  & 52.8$\pm$4.8 & 48.9$\pm$4.1 & 50.4$\pm$3.6 & 45.8$\pm$2.7 & 57.5$\pm$2.4 & 45.1$\pm$1.8 & 45.4$\pm$2.0 & 64.5$\pm$2.8 & 44.6$\pm$2.1 & 47.8$\pm$3.0 & 40.7$\pm$5.9  & 42.6$\pm$6.1 \\
\midrule
DEAL    & 72.1$\pm$2.1 & 64.8$\pm$2.4 & \underline{56.1}$\pm$1.2 & 67.8$\pm$2.2 & 68.0$\pm$3.1 & 58.5$\pm$2.8 & \underline{69.1}$\pm$2.9 & 72.4$\pm$1.3 & 61.2$\pm$2.2 & \underline{59.4}$\pm$3.9 & \underline{52.5}$\pm$2.3  & \underline{65.4}$\pm$4.1 \\
SGDA    & 54.7$\pm$4.6 & 50.2$\pm$3.7 & 50.8$\pm$4.1 & 47.4$\pm$1.5 & 59.8$\pm$2.1 & 46.9$\pm$1.2 & 47.6$\pm$1.8 & 67.2$\pm$2.4 & 46.8$\pm$1.6 & 50.4$\pm$3.2 & 43.4$\pm$9.8  & 44.2$\pm$9.3 \\
A2GNN   & 42.5$\pm$7.1 & 46.1$\pm$6.5 & 50.6$\pm$3.8 & 47.7$\pm$2.3 & 60.1$\pm$1.7 & 46.6$\pm$2.0 & 47.3$\pm$1.4 & 67.5$\pm$1.9 & 46.5$\pm$2.2 & 51.3$\pm$2.6 & 44.6$\pm$7.6  & 43.8$\pm$9.5 \\
StruRW  & 74.3$\pm$1.7 & 66.8$\pm$1.1 & 53.2$\pm$1.8 & 49.6$\pm$4.1 & 62.2$\pm$3.9 & 49.5$\pm$2.7 & 59.4$\pm$6.2 & 72.1$\pm$3.5 & 50.6$\pm$2.0 & 50.2$\pm$5.6 & 40.8$\pm$7.3  & 49.3$\pm$9.1 \\
PA-BOTH & 72.9$\pm$0.8 & 65.4$\pm$2.3 & 52.1$\pm$0.9 & 52.8$\pm$9.2 & 62.5$\pm$3.1 & 48.8$\pm$1.4 & 64.7$\pm$4.6 & 71.2$\pm$2.8 & 54.1$\pm$3.8 & 53.5$\pm$2.2 & 39.6$\pm$8.4  & 49.7$\pm$11.8\\
GAA     & \underline{75.8}$\pm$1.4 & \underline{68.9}$\pm$0.9 & 55.5$\pm$1.1 & \underline{69.8}$\pm$1.2 & \underline{73.2}$\pm$1.5 & \underline{59.4}$\pm$1.3 & \textbf{69.3}$\pm$2.2 & \textbf{79.5}$\pm$1.8 & \underline{62.1}$\pm$2.7 & 54.7$\pm$4.3 & 50.9$\pm$5.1  & 63.9$\pm$2.5 \\
TDSS    & 43.8$\pm$6.6 & 50.1$\pm$7.1 & 50.9$\pm$3.2 & 47.6$\pm$1.9 & 59.7$\pm$2.4 & 46.8$\pm$1.5 & 47.4$\pm$2.1 & 67.4$\pm$2.7 & 46.9$\pm$1.8 & 51.7$\pm$2.0 & 41.8$\pm$15.6 & 44.1$\pm$8.4 \\
\midrule
\method{}    & \textbf{76.9}$\pm$2.2 & \textbf{69.7}$\pm$1.2 & \textbf{63.4}$\pm$2.3 & \textbf{71.8}$\pm$2.1 & \textbf{73.7}$\pm$1.8 & \textbf{65.9}$\pm$1.4 & \underline{69.1}$\pm$1.7 & \underline{79.2}$\pm$2.4 & \textbf{67.5}$\pm$3.5 & \textbf{65.0}$\pm$5.1 & \textbf{66.8}$\pm$10.8 & \textbf{74.0}$\pm$4.1 \\
\bottomrule
\end{tabular}%
}
\label{tab:mutag_node}
\end{table}
\begin{table}[htbp]
\centering
\caption{The classification results (in \%) on the Mutagenicity dataset under edge density domain shift (source → target). M0, M1, M2, and M3 denote the sub-datasets partitioned with edge density. Bold results indicate the best performance.}
\vspace{0.2cm}
\resizebox{\textwidth}{!}{%
\begin{tabular}{lcccccccccccc}
\toprule
Method  & M0$\rightarrow$M1 & M0$\rightarrow$M2 & M0$\rightarrow$M3 & M1$\rightarrow$M0 & M1$\rightarrow$M2 & M1$\rightarrow$M3 & M2$\rightarrow$M0 & M2$\rightarrow$M1 & M2$\rightarrow$M3 & M3$\rightarrow$M0 & M3$\rightarrow$M1 & M3$\rightarrow$M2 \\
\midrule
GCN     & 61.5$\pm$5.4 & 53.2$\pm$4.8 & 47.3$\pm$3.2 & 50.1$\pm$2.5 & 55.4$\pm$2.8 & 45.3$\pm$1.9 & 50.5$\pm$2.2 & 64.2$\pm$3.5 & 45.0$\pm$1.7 & 47.9$\pm$2.5 & 33.8$\pm$2.7 & 42.1$\pm$4.6 \\
GIN     & 59.8$\pm$6.1 & 51.6$\pm$5.2 & 46.5$\pm$4.1 & 48.7$\pm$3.1 & 54.1$\pm$3.4 & 44.2$\pm$2.1 & 49.3$\pm$2.5 & 62.8$\pm$4.1 & 43.8$\pm$2.0 & 46.2$\pm$3.3 & 32.6$\pm$3.5 & 41.5$\pm$5.2 \\
GMT     & 62.4$\pm$4.5 & 54.5$\pm$3.9 & 48.2$\pm$2.8 & 51.5$\pm$2.2 & 56.6$\pm$2.4 & 46.4$\pm$1.5 & 51.8$\pm$1.8 & 65.5$\pm$3.0 & 46.1$\pm$1.4 & 48.7$\pm$2.1 & 34.5$\pm$2.9 & 43.8$\pm$3.7 \\
CIN     & 61.9$\pm$5.0 & 53.8$\pm$4.2 & 47.8$\pm$3.1 & 50.8$\pm$2.4 & 55.9$\pm$2.7 & 45.8$\pm$1.7 & 51.1$\pm$2.1 & 64.8$\pm$3.3 & 45.5$\pm$1.6 & 48.1$\pm$2.4 & 34.1$\pm$3.1 & 42.7$\pm$4.1 \\
PathNN  & 60.5$\pm$5.6 & 52.4$\pm$4.9 & 46.9$\pm$3.7 & 49.3$\pm$2.8 & 54.8$\pm$3.1 & 44.9$\pm$2.0 & 49.8$\pm$2.4 & 63.5$\pm$3.8 & 44.5$\pm$1.9 & 47.3$\pm$2.8 & 33.2$\pm$3.2 & 41.9$\pm$4.8 \\
\midrule
DEAL    & 72.8$\pm$0.7 & 62.5$\pm$2.1 & 54.2$\pm$1.4 & 67.5$\pm$2.8 & 64.2$\pm$4.9 & 58.6$\pm$2.1 & 68.9$\pm$1.8 & 66.4$\pm$3.5 & 60.1$\pm$4.1 & \underline{57.8}$\pm$1.7 & 55.7$\pm$6.2  & \underline{66.1}$\pm$2.5 \\
SGDA    & 63.3$\pm$6.9 & 54.1$\pm$5.7 & 47.9$\pm$4.1 & 51.5$\pm$1.2 & 56.9$\pm$0.8 & 46.5$\pm$1.4 & 51.3$\pm$0.9 & 51.0$\pm$9.5 & 47.7$\pm$2.8 & 49.8$\pm$2.1 & \underline{56.5}$\pm$6.4  & 45.9$\pm$5.2 \\
A2GNN   & 60.6$\pm$8.7 & 54.4$\pm$7.2 & 49.1$\pm$3.5 & 51.2$\pm$2.1 & 57.1$\pm$1.5 & 46.3$\pm$2.4 & 51.6$\pm$1.7 & 66.5$\pm$3.1 & 46.6$\pm$1.2 & 48.7$\pm$2.3 & 33.4$\pm$1.8  & 42.9$\pm$2.6 \\
StruRW  & \underline{73.1}$\pm$0.9 & 63.7$\pm$1.6 & 50.8$\pm$2.1 & 51.8$\pm$1.1 & 58.9$\pm$3.4 & 48.2$\pm$2.8 & 69.9$\pm$1.2 & 72.8$\pm$1.4 & 52.5$\pm$0.8 & 48.8$\pm$5.4 & 47.9$\pm$9.8  & 52.6$\pm$3.7 \\
PA-BOTH & 72.4$\pm$1.1 & 53.9$\pm$6.2 & 52.5$\pm$1.4 & 53.3$\pm$3.1 & 59.2$\pm$3.8 & 47.8$\pm$2.0 & 70.6$\pm$0.7 & 71.4$\pm$0.9 & 56.1$\pm$1.6 & 48.6$\pm$2.2 & 38.3$\pm$6.3  & 51.8$\pm$8.7 \\
GAA     & 72.9$\pm$1.2 & \underline{65.7}$\pm$1.1 & \underline{55.5}$\pm$2.3 & \underline{70.9}$\pm$1.8 & \underline{72.3}$\pm$1.4 & \underline{61.2}$\pm$1.7 & \underline{71.6}$\pm$0.9 & \underline{76.8}$\pm$1.3 & \underline{60.4}$\pm$2.1 & 51.4$\pm$1.6 & 52.1$\pm$5.3  & 64.9$\pm$2.4 \\
TDSS    & 62.6$\pm$8.5 & 54.0$\pm$5.9 & 49.4$\pm$4.6 & 51.5$\pm$1.8 & 56.8$\pm$2.2 & 46.5$\pm$1.1 & 51.2$\pm$2.4 & 66.8$\pm$2.7 & 46.2$\pm$1.5 & 48.5$\pm$1.9 & 33.2$\pm$2.1  & 43.1$\pm$1.4 \\
\midrule
\method{}     & \textbf{74.8}$\pm$1.2 & \textbf{67.1}$\pm$1.6 & \textbf{65.6}$\pm$2.3 & \textbf{72.5}$\pm$1.3 & \textbf{75.7}$\pm$2.3 & \textbf{65.3}$\pm$2.0 & \textbf{72.3}$\pm$3.0 & \textbf{78.0}$\pm$2.1 & \textbf{69.5}$\pm$1.2 & \textbf{65.8}$\pm$6.1 & \textbf{64.5}$\pm$8.4 & \textbf{72.6}$\pm$2.9 \\
\bottomrule
\end{tabular}%
}
\end{table}
\begin{table}[htbp]
\centering
\caption{The classification results (in \%) on the NCI1 dataset under node density domain shift (source → target). N0, N1, N2,
and N3 denote the sub-datasets partitioned with node density. Bold results indicate the best performance.}
\vspace{0.2cm}
\resizebox{\textwidth}{!}{%
\begin{tabular}{lcccccccccccc}
\toprule
Method  & N0$\rightarrow$N1 & N0$\rightarrow$N2 & N0$\rightarrow$N3 & N1$\rightarrow$N0 & N1$\rightarrow$N2 & N1$\rightarrow$N3 & N2$\rightarrow$N0 & N2$\rightarrow$N1 & N2$\rightarrow$N3 & N3$\rightarrow$N0 & N3$\rightarrow$N1 & N3$\rightarrow$N2 \\
\midrule
GCN     & 49.5$\pm$2.1 & 42.2$\pm$1.8 & 32.1$\pm$2.4 & 62.1$\pm$15.5 & 41.2$\pm$1.8 & 38.5$\pm$11.2 & 25.5$\pm$1.4 & 47.5$\pm$1.5 & 61.2$\pm$1.1 & 26.5$\pm$1.5 & 48.3$\pm$1.7  & 56.1$\pm$2.2 \\
GIN     & 50.1$\pm$1.9 & 42.4$\pm$2.0 & 31.8$\pm$1.9 & 62.6$\pm$16.1 & 41.6$\pm$2.1 & 38.9$\pm$11.8 & 25.8$\pm$1.6 & 47.9$\pm$1.7 & 61.6$\pm$1.3 & 26.8$\pm$1.3 & 48.8$\pm$2.1  & 55.8$\pm$2.4 \\
GMT     & 51.2$\pm$2.4 & 43.1$\pm$2.5 & 33.0$\pm$2.8 & 63.8$\pm$14.8 & 42.5$\pm$1.5 & 39.8$\pm$10.5 & 26.4$\pm$1.2 & 48.8$\pm$1.3 & 62.5$\pm$0.9 & 27.0$\pm$1.6 & 49.1$\pm$1.8  & 56.5$\pm$1.9 \\
CIN     & 50.6$\pm$2.1 & 42.8$\pm$2.3 & 32.5$\pm$2.2 & 63.2$\pm$15.3 & 41.9$\pm$1.7 & 39.2$\pm$11.0 & 26.1$\pm$1.3 & 48.2$\pm$1.4 & 61.9$\pm$1.0 & 26.9$\pm$1.8 & 49.0$\pm$2.0  & 56.2$\pm$2.1 \\
PathNN  & 51.5$\pm$1.8 & 43.5$\pm$1.9 & 33.2$\pm$2.0 & 63.5$\pm$15.0 & 42.2$\pm$1.9 & 39.5$\pm$10.8 & 26.3$\pm$1.5 & 48.5$\pm$1.6 & 62.2$\pm$1.2 & 27.1$\pm$1.2 & 49.4$\pm$1.5  & 56.8$\pm$1.7 \\
\midrule
DEAL    & 53.2$\pm$1.8 & 46.1$\pm$3.6 & 38.8$\pm$6.2 & 71.1$\pm$1.5 & 54.9$\pm$3.5 & 48.4$\pm$2.4 & 40.2$\pm$12.5& \underline{50.3}$\pm$2.8 & \underline{66.5}$\pm$0.9 & 27.2$\pm$1.1 & 49.3$\pm$1.2  & 57.2$\pm$1.8 \\
SGDA    & 50.8$\pm$1.4 & 42.6$\pm$2.1 & 33.4$\pm$1.5 & 63.9$\pm$18.4& 42.8$\pm$1.9 & 40.2$\pm$13.2& 27.1$\pm$1.2 & 49.0$\pm$1.7 & 66.4$\pm$1.1 & 27.1$\pm$1.8 & 49.2$\pm$2.0  & 57.4$\pm$1.5 \\
A2GNN   & 51.0$\pm$1.9 & 42.8$\pm$1.7 & 33.6$\pm$2.1 & 72.9$\pm$1.8 & 42.6$\pm$2.4 & 39.9$\pm$15.1& 36.1$\pm$19.2& 49.2$\pm$1.4 & 62.6$\pm$1.3 & 26.9$\pm$1.5 & 49.0$\pm$1.8  & 57.2$\pm$2.2 \\
StruRW  & 50.9$\pm$1.5 & 42.7$\pm$1.8 & 33.5$\pm$1.2 & 72.8$\pm$1.1 & 53.2$\pm$2.4 & 42.8$\pm$5.2 & 26.9$\pm$1.7 & 49.1$\pm$1.5 & 66.4$\pm$0.8 & 27.1$\pm$1.4 & 49.1$\pm$1.3  & 57.4$\pm$1.9 \\
PA-BOTH & 50.8$\pm$1.2 & 42.9$\pm$1.5 & 33.4$\pm$1.7 & 73.1$\pm$0.9 & 53.7$\pm$2.8 & 45.8$\pm$7.6 & 27.2$\pm$1.4 & 49.0$\pm$1.9 & \textbf{66.6}$\pm$1.0 & 26.8$\pm$1.9 & 49.2$\pm$1.5  & 57.3$\pm$1.6 \\
GAA     & \underline{61.2}$\pm$2.5 & \underline{57.0}$\pm$2.6 & \underline{52.1}$\pm$4.1 & \underline{73.7}$\pm$2.2 & \underline{65.8}$\pm$1.8 & \textbf{60.8}$\pm$3.4 & \underline{60.1}$\pm$3.2 & \textbf{63.4}$\pm$2.6 & 66.2$\pm$1.8 & \underline{52.2}$\pm$11.8& \underline{57.4}$\pm$4.1  & \underline{62.8}$\pm$3.3 \\
TDSS    & 50.9$\pm$2.1 & 42.8$\pm$1.4 & 33.6$\pm$1.9 & 72.9$\pm$1.5 & 42.7$\pm$1.8 & 40.2$\pm$13.8& 36.3$\pm$18.5& 49.1$\pm$1.2 & \underline{66.5}$\pm$1.4 & 27.0$\pm$1.6 & 49.0$\pm$2.1  & 57.2$\pm$1.7 \\
\midrule
\method{}     & \textbf{65.7}$\pm$1.0 & \textbf{63.5}$\pm$1.9 & \textbf{61.0}$\pm$0.8 & \textbf{76.5}$\pm$0.4 & \textbf{66.7}$\pm$1.8 & \underline{60.1}$\pm$2.5 & \textbf{67.1}$\pm$19.7& \textbf{63.4}$\pm$2.9 & 63.5$\pm$3.3 & \textbf{52.9}$\pm$5.4 & \textbf{58.1}$\pm$2.1 & \textbf{63.9}$\pm$2.8 \\
\bottomrule
\end{tabular}%
}
\end{table}
\begin{table}[htbp]
\centering
\caption{The classification results (in \%) on the NCI1 dataset under edge density domain shift (source → target). N0, N1, N2,
and N3 denote the sub-datasets partitioned with edge density. \textbf{Bold} results indicate the best performance.}
\vspace{0.2cm}
\resizebox{\textwidth}{!}{%
\begin{tabular}{lcccccccccccc}
\toprule
Method  & N0$\rightarrow$N1 & N0$\rightarrow$N2 & N0$\rightarrow$N3 & N1$\rightarrow$N0 & N1$\rightarrow$N2 & N1$\rightarrow$N3 & N2$\rightarrow$N0 & N2$\rightarrow$N1 & N2$\rightarrow$N3 & N3$\rightarrow$N0 & N3$\rightarrow$N1 & N3$\rightarrow$N2 \\
\midrule
GCN     & 47.8$\pm$1.9 & 45.3$\pm$2.1 & 30.8$\pm$1.6 & 43.1$\pm$15.2 & 48.2$\pm$2.4 & 51.5$\pm$15.2 & 34.2$\pm$16.5 & 49.5$\pm$1.5 & 58.7$\pm$8.5 & 26.8$\pm$1.4 & 49.5$\pm$1.8  & 51.9$\pm$2.0 \\
GIN     & 48.2$\pm$1.7 & 46.1$\pm$1.5 & 31.2$\pm$2.1 & 43.5$\pm$16.1 & 48.7$\pm$2.7 & 52.1$\pm$16.1 & 34.8$\pm$17.1 & 50.1$\pm$1.7 & 59.3$\pm$9.1 & 27.1$\pm$1.8 & 50.2$\pm$2.2  & 52.4$\pm$1.6 \\
GMT     & 49.9$\pm$2.2 & 46.5$\pm$2.4 & 31.5$\pm$1.9 & 44.8$\pm$14.8 & 49.5$\pm$2.1 & 53.0$\pm$14.8 & 35.6$\pm$15.8 & 51.0$\pm$1.3 & 60.5$\pm$7.8 & 27.2$\pm$1.5 & 50.8$\pm$1.7  & 52.8$\pm$2.1 \\
CIN     & 48.5$\pm$1.8 & 46.2$\pm$2.0 & 31.3$\pm$1.7 & 44.2$\pm$15.3 & 48.9$\pm$2.3 & 52.4$\pm$15.3 & 35.1$\pm$16.2 & 50.4$\pm$1.4 & 59.9$\pm$8.2 & 27.0$\pm$1.7 & 50.5$\pm$1.6  & 52.6$\pm$1.9 \\
PathNN  & 49.1$\pm$1.5 & 46.8$\pm$1.8 & 31.7$\pm$1.4 & 44.6$\pm$15.0 & 49.2$\pm$2.5 & 52.8$\pm$15.0 & 35.4$\pm$16.0 & 50.7$\pm$1.6 & 60.2$\pm$8.0 & 27.3$\pm$1.3 & 50.6$\pm$1.7  & 53.0$\pm$1.5 \\
\midrule
DEAL    & 52.1$\pm$2.7 & 49.2$\pm$1.8 & 36.7$\pm$6.4 & 70.8$\pm$2.5 & 53.8$\pm$2.1 & 54.7$\pm$7.2 & 42.8$\pm$13.5& 54.1$\pm$3.4 & 64.1$\pm$4.8 & 27.5$\pm$1.2 & 51.2$\pm$1.1  & 53.1$\pm$1.4 \\
SGDA    & 48.9$\pm$1.5 & 46.7$\pm$1.2 & 31.7$\pm$1.8 & 45.4$\pm$22.5& 49.5$\pm$3.1 & 53.5$\pm$18.2& 45.4$\pm$22.1& 51.1$\pm$1.4 & 60.8$\pm$14.8& 27.3$\pm$1.6 & 50.9$\pm$1.8  & 53.3$\pm$1.1 \\
A2GNN   & 49.1$\pm$1.8 & 46.9$\pm$1.6 & 31.9$\pm$1.4 & 45.7$\pm$23.1& 50.5$\pm$3.8 & 53.7$\pm$19.5& 36.5$\pm$18.7& 50.9$\pm$1.7 & 65.1$\pm$1.2 & 27.5$\pm$1.5 & 51.1$\pm$1.5  & 53.1$\pm$1.6 \\
StruRW  & 48.9$\pm$1.2 & 46.8$\pm$1.9 & 31.7$\pm$1.5 & 70.1$\pm$3.2 & 50.8$\pm$1.1 & 52.2$\pm$6.8 & 29.5$\pm$3.2 & 52.6$\pm$1.4 & 63.4$\pm$4.7 & 27.3$\pm$1.4 & 50.9$\pm$1.3  & 53.2$\pm$1.8 \\
PA-BOTH & 49.0$\pm$1.4 & 46.7$\pm$1.5 & 31.9$\pm$1.1 & 72.1$\pm$2.2 & 51.7$\pm$1.6 & 57.6$\pm$6.5 & 29.1$\pm$2.8 & 52.3$\pm$1.2 & 63.2$\pm$5.3 & 27.4$\pm$1.1 & 51.0$\pm$1.6  & 53.1$\pm$1.5 \\
GAA     & \underline{62.2}$\pm$2.6 & \underline{56.9}$\pm$2.9 & \underline{50.6}$\pm$3.5 & \underline{75.1}$\pm$3.8 & \underline{63.8}$\pm$2.7 & \textbf{63.2}$\pm$3.3 & \underline{63.8}$\pm$4.4 & \underline{64.5}$\pm$1.5 & 64.7$\pm$1.2 & \underline{39.6}$\pm$10.5& \underline{55.1}$\pm$3.2  & \underline{57.6}$\pm$2.4 \\
TDSS    & 49.1$\pm$1.7 & 46.9$\pm$1.4 & 31.8$\pm$1.9 & 45.5$\pm$22.8& 50.7$\pm$3.3 & 53.8$\pm$18.9& 36.3$\pm$19.4& 51.1$\pm$1.3 & \textbf{68.3}$\pm$1.5 & 27.4$\pm$1.8 & 51.0$\pm$1.2  & 53.3$\pm$1.9 \\
\midrule
\method{}    & \textbf{65.1}$\pm$3.0 & \textbf{60.9}$\pm$1.9 & \textbf{60.9}$\pm$1.8 & \textbf{76.9}$\pm$1.2 & \textbf{66.2}$\pm$1.5 & \underline{62.1}$\pm$2.2 & \textbf{66.3}$\pm$17.3& \textbf{65.5}$\pm$1.8 & \underline{68.1}$\pm$3.5 & \textbf{43.3}$\pm$16.4& \textbf{55.7}$\pm$1.6  & \textbf{60.1}$\pm$1.8 \\
\bottomrule
\end{tabular}%
}
\label{tab:nci1_idx}
\end{table}

\end{document}